\documentclass[11pt]{article}

\usepackage{sustyle}
\usepackage[utf8]{inputenc}
\usepackage{booktabs}
\usepackage[numbers]{natbib}
\usepackage[colorlinks=true,allcolors=blue]{hyperref}
\usepackage{url}
\usepackage{array}
\usepackage{graphicx}
\graphicspath{{./figures/}}
\usepackage{enumitem}
\usepackage[ruled, vlined]{algorithm2e}
\usepackage{bm}
\usepackage{color, float}
\usepackage{tcolorbox}
\usepackage{caption}
\usepackage{subcaption}

\newcommand{\bS}{\bm{S}}

\newcommand{\norm}[1]{\left\|#1\right\|}

\renewcommand{\cite}{\citep}

\newcommand{\D}{\mathcal{D}}

\newcommand{\set}[1]{\left\{#1\right\}}

\newcommand{\Si}{S^{(i)}}
\newcommand{\Sj}{S^{(j)}}

\newcommand{\Sjj}{S^{'(j)}}
\newcommand{\wi}{w^{(i)}}
\newcommand{\wj}{w^{(j)}}

\newcommand{\hinoise}{\tilde{h}^{(i)}}

\newcommand{\hjnoise}{\tilde{h}^{(j)}}

\newcommand{\Htilde}{\tilde{H}}
\newcommand{\wglobal}{w^{\text{global}}}
\newcommand{\wgnoise}{\tilde{w}^{\text{global}}}
\newcommand{\winoise}{\tilde{w}^{(i)}}
\newcommand{\wjnoise}{\tilde{w}^{(j)}}
\newcommand{\fedavg}{\texttt{FedAvg}\xspace}

\newcommand{\fedsync}{\texttt{PriFedSync}\xspace}

\newcommand{\sample}{\texttt{Sample}_p}
\newcommand{\bSj}{\bS^{'j}}
\newcommand{\zero}{\bm{0}}
\newcommand{\Id}{\text{Id}}
\newcommand{\dpsgd}{\texttt{DP-SGD}\xspace}

\title{Federated $f$-Differential Privacy}

\author{Qinqing Zheng\thanks{Department of Statistics. Email: {\tt zhengqinqing@gmail.com}.}
	\and Shuxiao Chen\thanks{Department of Statistics. Email: {\tt shuxiaoc@wharton.upenn.edu}. }
	\and Qi Long\thanks{Department of Biostatistics, Epidemiology and Informatics. Email: {\tt qlong@pennmedicine.upenn.edu}.}
	\and Weijie J.~Su\thanks{Department of Statistics. Email: {\tt suw@wharton.upenn.edu}.}
}

\date{}
\begin{document}
\maketitle
\begin{center}
\vskip-23pt
University of Pennsylvania\\
\vskip10pt
February 18, 2021
\end{center}

\begin{abstract}
    Federated learning (FL) is a training paradigm where the clients collaboratively learn
models by repeatedly sharing information without compromising much on the privacy of
their local sensitive data. 
%
%
In this paper, we introduce
\emph{federated $f$-differential privacy}, a new notion specifically tailored to the federated
setting, based on the framework of Gaussian differential privacy. Federated $f$-differential privacy
operates on \emph{record level}: it provides the privacy guarantee on each individual
record of one client's data against adversaries.  We then propose a generic private
federated learning framework \fedsync that accommodates a large family of
state-of-the-art FL algorithms, which provably achieves {federated $f$-differential privacy}.
Finally, we empirically demonstrate the trade-off between privacy guarantee and
prediction performance for models trained by \fedsync in computer vision tasks.

\end{abstract}

\section{Introduction}
\emph{Federated learning} \cite{mcmahan2016communication} is an emerging paradigm that enables
multiple clients to collaboratively learn prediction models without explicitly sharing
data.  Unlike traditional distributed training approaches
that upload all the data to central servers, federated learning performs on-device training and
only some summaries of local data or local models are exchanged among clients. Typically,
the clients upload their local models to the server and share the global averaging in a repeated manner.
This offers plausible solutions to address the critical data privacy issue:
sensitive information about individuals such as typing history, shopping transactions,
geographical locations, medical records, would stay localized. 

Nonetheless, a malicious client who participates in the federated learning
might still be able to learn information about the other clients' data
through the shared model's weights. This is because
it is possible for an adversary to learn about or even identify
certain individuals by simply tweaking the input datasets and probing the output of the algorithm \cite{fredrikson2015model, shokri2017membership}.
This gives rise to a pressing call for privacy-preserving federated learning algorithms.
Accordingly, we urgently need a rigorous and principled framework to enhance data privacy,
and to quantitatively answer the important questions:
\begin{center}
{\em Can another client identify the presence or absence of any individual record in my data in federated learning? Worse, what if all the other clients ally each other to attack my data?}
\end{center}
A number of works have tried to answer \textit{similar} questions from different perspectives
\cite{mcmahan2017learning, geyer2017differentially, li2019privatemeta, singh2020differentially},
and numerous privacy notions and associated approaches are proposed to address \textit{related} problems,
yet \textit{none} of them have answered these questions fully and directly.
To the best of our knowledge, all existing works plainly generalize the classical differential privacy
definition to the federated setting: an adversary can remove one client's whole dataset, and this type of attack would not incur massive changes to the output of the algorithm. The resulting privacy guarantee executes at the user level: whether a client has participated in the training
can not be inferred by adversaries, and the client's whole dataset is private.

While the user level privacy has important applications in federated learning,
it is complementary and equally important to consider weaker privacy notions at the
record level.
First, privacy is generally at odds with performance. A user level privacy guarantee is usually too strong and one often seeks a weaker notion that protects privacy from more practical attacks \cite{li2019privatemeta}.
More importantly,
consider the case whee multiple hospitals in different countries would like to collaboratively
learn prediction models for COVID-19. In this example, whether a hospital participates
in this collaboration is not a sensitive information at all,
and what really needs to be protected is the privacy of each patient. This is
a regime that a record level privacy notion shines.

In this work, we introduce a fine-grained privacy notion, called \emph{weak federated $f$-differential privacy},
that protects \emph{each individual record} of one client's data. We work
on the attack model that an adversary can manipulate one single record of the client's
dataset and provide privacy guarantees for this case.  We propose a unified private
federated learning framework \fedsync where a large family of federated learning
algorithms kick in.  Besides, we give an extended privacy notion, called \emph{strong federated $f$-differential privacy}, to address the case where multiple malicious clients jointly
attack a client, which has not been considered in any previous work.  Our major
contributions are as follows.

\begin{enumerate}
    \item We introduce two privacy notions, \emph{weak federated $f$-differential privacy} and
        \emph{strong federated $f$-differential privacy}, that describe the privacy guarantee against
        an individual adversary and against a group of adversaries, respectively. Both
        notions are of the finest resolution in the sense that they protect individual
        records of one client's data. The privacy definition that we rely on is $f$-differential privacy, in particular its sub-family of \emph{Gaussian differential privacy} (GDP)~\citep{dong2019gaussian}.

    \item We propose a generic federated learning framework \fedsync that contains the
        state-of-the-art federated learning algorithms.  The framework does not assume a
        trusted central aggregator.  It can accommodate both personalized or non-personalized
        approaches. We exploit the composition theorem of GDP to analyze the privacy
        guarantee of \fedsync and prove its asymptotic convergence.

    \item We conduct numerical experiments to illustrate our privacy notions and compare
        the performance of private models with non-private counterparts. When the data
        is heterogeneous across clients, our personalized approach demonstrates
        significant improvement over the global model. We also demonstrate the
        trade-offs between privacy and accuracy, and privacy and computation, through
        our experiments.
\end{enumerate}

%
%
%

The rest of this paper is organized as follows.
We give a brief review of the research on federated learning and differential privacy in Section~\ref{sec:related}.
Section~\ref{sec:privacy_algo} introduces our training framework. Section~\ref{sec:privacy}
presents the privacy notion and analysis.
Section~\ref{sec:expr}  presents the numerical experiments.

\subsection{Related Work}
\label{sec:related}
There is a growing body of work that have looked at privacy properties in the context of federated learning.
\cite{mcmahan2017learning} introduces two algorithms, \emph{differentially private federated stochastic gradient descent} (\texttt{DP-FedSGD}) and \emph{differentially private federated averaging} (\texttt{DP-FedAvg}), and studies their privacy properties. The privacy notion is defined on user level. Namely, two datasets $S$ and $S'$
are said to be neighboring if $S'$ can be obtained by completely removing one client's data from $S$.
Such an attack might be impractical for real world applications.
Algorithmically, \texttt{DP-FedSGD} is a direct extension of ``non-federated'' \texttt{DP-SGD} \cite{abadi2016deep} to the distributed optimization setting, where the gradients
of each client is clipped and aggregated in every iteration, whereas \texttt{DP-FedAvg} performs approximated \texttt{DP-SGD} on the server. In essence, the differences of local models before and after local training are treated as the surrogates of gradients and sent to the server.
A similar algorithm approximating \texttt{DP-SGD} is proposed in \cite{geyer2017differentially}. \cite{singh2020differentially} uses an algorithm similar to \texttt{DP-FedSGD} for the architecture search problem, and their privacy guarantee
acts on user level too. \cite{li2019privatemeta} studies the online transfer learning and introduces a notion called task global privacy that works on record level. However, the online setting assumes the client only interacts with the server once and does not extend to the federated setting. 
\cite{truex2019hybrid} generalizes the central differential privacy into the distributed
setting, which can be considered as training a single shared private model when the
dataset is split into several partitions on different machines. Even though the authors
consider privacy at record level, this work does not have setups such as individual clients
have their own models, client sampling, etc.

One privacy notion that is related to the general concept of privacy in
federated learning is local differential privacy
\cite{evfimievski2003limiting, localdp}. Local differential privacy does not assume a trusted data
aggregator. Each data record is randomly perturbed before sending to the
data aggregator, and the aggregator build models using the noisy data.
The perturbation algorithm is locally differentially private
if the outputs of any pair of possible data records are indistinguishable.
Although conceptual connected, local differential privacy does not perfectly extend to the general
federated learning environment. Under the local differential privacy framework, the noisy data are finally
centralized in a central aggregator, where all the training happens; whereas a general form of
federated learning allows the participants have their own control of data and models.
Besides, local differential privacy is a strong notion that often requires a large amount of noise and thus
leads to degraded model performance.

Another related notion is joint differential privacy, proposed by \cite{kearns2014mechanism} to
study the behavior of ``recommender mechanisms'' for large games. It has been applied to the
context of private convex programming for problems
whose solution can be divided between different agents \cite{hsu2016private,
hsu2016jointly}, e.g. the multi-commodity flow problem.
Informally, joint differential privacy ensures the joint distribution of the outputs
for Agent $j\neq i$ to be insensitive to the input provided by Agent $i$. It is
similar to the our one-vs.-all notion \emph{strong federated $f$-differential privacy} (see
Definition~\ref{def:privacy_strong}), but acts on user
level.

Despite the granularity and concrete notion of the privacy guarantee, a formal privacy definition is needed to precisely
quantify the privacy loss. The most popular statistical privacy definition to date is $(\epsilon,
\delta)$-differential privacy \cite{Dwork2006, Dwork2006_2}.
It is widely applied in industrial applications and academic research, including some previous work on private federated learning \cite{mcmahan2016communication, geyer2017differentially, li2019privatemeta}.
Unfortunately, this privacy definition does not well handle
the cumulative privacy loss under the composition of private algorithms \cite{kairouz2017composition, complexity},
which is a fundamental problem to address in privacy analysis, and also
needed in analyzing the federated learning algorithms.
The need for a better treatment of composition has motivated much work in
proposing divergence-based relaxations of $(\epsilon, \delta)$-differential privacy
relaxations \cite{concentrated,concentrated2,mironov2017renyi,tcdp}.
Meanwhile, another line of research has established the connection between differential privacy and
hypothesis testing \cite{wasserman2010statistical, kairouz2017composition, liu2019investigating, balle2019hypothesis}.
Recently, \cite{dong2019gaussian} proposes a hypothesis testing-based privacy notion termed $f$-\textit{differential privacy}. This privacy definition
characterizes the privacy guarantee using
the trade-off between type I and type II errors given via the associated hypothesis testing problem. In the case of testing for normal distributions, $f$-differential privacy reduces to {Gaussian differential privacy}. Owing to its lossless reasoning about composition and privacy amplification by subsampling, the use of $f$-differential privacy gives sharp, analytically tractable expressions for the privacy guarantees of training deep learning models \cite{bu2019deep} (see also \cite{zheng2020sharp}). Throughout this paper, we use GDP as our privacy analysis framework.

\section{Private Federated Learning}
\label{sec:privacy_algo}
Let $m$ denote the number of clients. Each Client $i$ has access to its local
dataset $\Si$, where the data are i.i.d sampled from local
distribution $\D_i$. The classic federated learning algorithms \cite{mcmahan2016communication,konevcny2016federated} aim at learning one global model $\wgnoise$
that performs well over all the clients.
This implicitly makes an underlying assumption that the data are homogeneous, i.e., $\D_1 = \cdots = \D_m$, yet in practice
data might not be identically distributed across clients.
To take into account of the heterogeneity of user data distributions,
there is a surge of interest to assume non-identical data distributions with the possibility of $\D_i \neq \D_j$,
and learn personalized models
\cite{dinh2020moreau, huang2020Attentive, hanzely2020mixture, deng2020adaptive}.

We propose a unified framework \fedsync that addresses both heterogeneous and homogeneous settings, see Algorithm~\ref{alg:fedsync}.
Each Client $i$ will obtain a specific model $\winoise$, and a global model $\wgnoise$
is still formed and utilized. The homogeneous setting boils down to a special case where $\winoise = \wgnoise, i \in [m]$ (we use $[m]$ to denote $\{1, \hdots, m\}$).
\fedsync subsumes a large family of existing federated learning algorithms, including \fedavg \cite{mcmahan2016communication} and many others
\cite{li2018fedprox,dinh2020moreau,huang2020Attentive,hanzely2020mixture, deng2020adaptive}.

\begin{algorithm}[t]
\SetKw{kwInput}{Input:}
\SetKw{kwOutput}{Output:}
\SetKwBlock{DoParallel}{Client $i \in \Omega$ do in parallel}{}
\SetKwBlock{DoGlobal}{Server do}{}
\DontPrintSemicolon
\kwInput{initialization $w_0$, number of local iterations $K$, number of synchronization rounds $R$, sync probability $p$}\\
Initialization:$\wi, \hinoise \leftarrow w_0$\;
\For{r = 1, \ldots R}{
    Poisson sample a subset of clients $\Omega \subseteq \set{1, \ldots, m}$ with probability $p$ \;
    \DoParallel{
        \tcp{\small local training for $K$ iterations}
        $\winoise \leftarrow \text{LocalPrivateTraining}(\Si, \hinoise, K)$ \;
        Send $\winoise$ to Server\;
    }
    \DoGlobal{
        $\wgnoise \leftarrow (1 - \eta) \wgnoise + \eta \frac{1}{|\Omega|} \sum_{i\in \Omega} \winoise$\;
        $\hinoise \leftarrow F_i(\wgnoise), \; i \in \Omega$\;
        Push noisy helper model $\hinoise$ to Client $i$, $i \in \Omega$\;
    }
}
\caption{\fedsync Framework}
\label{alg:fedsync}
\end{algorithm}

%

In \fedsync, all the clients start from the same model $w_0$.
 To mimic the practical behavior that not all the
clients sync with the server simultaneously, in every synchronization round
we sample a subset of clients to perform local training
and sync with the server. If Client $i$ is selected, it pulls a helper model
$\hinoise$ from the server and then performs local private training for $K$ iterations.
The helper model $\hinoise$ can be the global aggregation $\wgnoise$ \cite{mcmahan2016communication, li2018fedprox}, or personalized \cite{dinh2020moreau, huang2020Attentive}. Various ways have been proposed to utilize $\hinoise$ to improve local training,
including initializing local models
\cite{dinh2020moreau, hanzely2020mixture}, regularizing local training
\cite{li2018fedprox,huang2020Attentive},
and iterative interpolating with local updates \cite{deng2020adaptive}.

The local private training can be carried out in different ways too. For instance, one can
use noiseless local training and perturb the model before synchronization using Laplacian or Gaussian mechanism.
Alternatively, one can conduct \dpsgd \cite{abadi2016deep} directly, where
the gradient is perturbed in each iteration. The disadvantage of \dpsgd is that it is slow in computation, due to its need
of clipping the per-sample gradient at every iteration.
However, we observed that \dpsgd usually leads to better prediction accuracy, therefore we shall use \dpsgd for our analysis and experiments.

Next, Client $i$ sends the private model $\winoise$ to the server.
The server aggregates the received models and then updates the corresponding helper models.
There are plenty of ways of computing the helper model. If the $F_i$ function is the identity map $F_i(w) = w, i \in [m]$, all the clients will receive
the same helper model $\wgnoise$. This is the setup used in \fedavg \cite{mcmahan2016communication}.
Another simple but effective observation is that $\hinoise$ is a
convex combination of the noisy global model $\wglobal$ and local model $\wi$ \cite{dinh2020moreau, hanzely2020mixture}:
\begin{equation}
    F_i(\wgnoise) = (1 - \alpha_i) \wi + \alpha_i \wgnoise.
    \label{eq:hi_intepolation}
\end{equation}
There are more sophisticated constructions of personalized helper models that fit in our framework, for example, an attention-based weighted averaging \cite{huang2020Attentive}.

We close this section with an overview discussion of the privacy guarantee of \fedsync.
\begin{enumerate}[wide, labelwidth=!, labelindent=0pt]
\item It is easy to see that Client $i$ can only probe the dataset of Client $j$ through the helper model $\hinoise$. This becomes the focal point for our analysis throughout this paper.

\item Given a global model $\wgnoise$, the helper model $\hinoise$ is a transformation of $\wgnoise$ through the mapping $F_i$.
    Regardless of the form of $F_i$, this step would not cause additional privacy leakage since differential privacy is immune to
    post-processing \cite{dwork2014algorithmic}. It is then natural to ask the following questions:

\begin{enumerate}
    \item[(i)] Why not just compute a noiseless global model
    $\wglobal$ and inject noise before or after the transformation? For instance, on the server one can conduct
    \[
    \begin{aligned}
    \wglobal & \leftarrow (1 - \eta) \wglobal + \eta \frac{1}{|\Omega|} \sum_{i\in \Omega} \wi, \\
    \hinoise & \leftarrow F_i\big(\wglobal + \N(0, \sigma^2 I) \big), i \in \Omega.
    \end{aligned}
    \]
    \item[(ii)]Is it equivalent to directly send $\wgnoise$ to the clients and let them apply the transformation $F_i$'s themselves?
\end{enumerate}
\end{enumerate}
    The procedure in (i) indeed protects the privacy of Client $i$ and reduces
    the computational burden incurred by \dpsgd. However, computing a noiseless global model will require all the clients send
    noiseless local models to the server, which
    imposes an extra assumption about a trustworthy server.
    The answer to (ii) depends on the concrete form of $F_i$.
    If the mapping $F_i$ is free of other private local models, deterministic, and invertible,
    the privacy cost before and after applying $F_i$ is the same. In this scenario, there is no difference between sending $\wgnoise$ or $\hinoise$. Nevertheless, post-processing might be able to amplify the privacy. Consider a constant function $F_i$ that outputs the zero vector for any input. This simple function achieves perfect privacy. For those cases,
    sending $\wgnoise$ will be less private then sending $\hinoise$.
    To keep our analysis general for all algorithms that fit in \fedsync, we shall assume no knowledge of $F_i$ in our analysis. For a specific algorithm, potential tighter bounds might be obtained by taking prior knowledge of $F_i$.

\section{Privacy Analysis}
\label{sec:privacy}
We first review Gaussian differential privacy in Section~\ref{sec:privacy_gdp}, which is the analysis tool we exploit.
Next, we introduce our private notations in Section~\ref{sec:privacy_notion} and analyze the privacy guarantee
of \fedsync in Section~\ref{sec:privacy_analysis}.

\subsection{Preliminaries}
\label{sec:privacy_gdp}
Let us start from the hypothesis testing interpretation of differential privacy, which is the foundation of GDP. Let $\A$ denote a randomized algorithm that takes a dataset $S$ as input. 
$S'$ is a neighboring dataset of $S$ in the sense that $S$ and $S'$ differ in only one individual.
Let $P$ and $Q$ denote the probability distribution of $\A(S)$ and $\A(S')$, respectively. 
Differential privacy attempts to measure the difficulty for an adversary to identify the presence or absence of any individual in $S$ via leveraging the output of $\A$. Equivalently, an adversary performs the following hypothesis testing problem \cite{wasserman2010statistical}:
\[
    H_0: \text{output} \, \sim P \;\; \text{vs} \;\; H_1: \text{output} \, \sim Q.
\]
Intuitively, a privacy breach occurs if the adversary makes the right decision, and
the privacy guarantee of $\A$ boils down to the difficulty for an adversary to tell the two distributions apart. 
\cite{dong2019gaussian} proposes to use the trade-off
between type I and type II errors of the optimal likelihood ratio tests 
at level $\alpha$ as a measure of the privacy guarantee,
where $\alpha$ ranges from $0$ to $1$. Formally, let $\phi$ be a rejection rule for testing against $H_0$ against $H_1$. The type I and type II error of $\phi$ are $\E_P(\phi)$ and $1 - \E_Q(\phi)$, respectively.
The trade-off function $T: [0,1] \rightarrow [0,1]$
between the two probability distributions $P$ and $Q$ is defined as
\[
T(P, Q)(\alpha) = \inf_{\phi} \set{1 - \E_Q(\phi): \E_P(\phi) \leq \alpha}.
\]
In short, for a fixed significance level $\alpha$, $T(P,Q)(\alpha)$ is the minimum type II error that a test can achieve
at that level. The optimal tests are given by the Neyman--Pearson lemma, and can be interpreted as the most powerful adversaries.
Let us define the relation $f \ge g$ if $f(\alpha) \ge g(\alpha)$ for all $0 \le \alpha \le 1$.
Intuitively speaking, a larger trade-off function implies the more private the associated algorithm is.
A special case of interest is when the two distributions are the same and perfect privacy is attained.
The corresponding trade-off function is $T(P,P)(\alpha) = 1 - \alpha$, which we denote by $\Id(\alpha)$.
With the above definitions in place,  \cite{dong2019gaussian} introduces the following privacy definition,
with a little abuse of notation by using $\A(S)$ to denote the output distribution of algorithm $\A$ on input dataset $S$. 
\begin{definition}\label{def:fdp}
Let $f$ be a trade-off function. An algorithm $\A$ is $f$-differentially private if $T(\A(S), \A(S')) \geq f$ for any pair of neighboring datasets $S$ and $S'$. 
\end{definition}
When the trade-off function is defined between two Gaussian distributions, we obtain a subfamily of $f$-differential privacy guarantees called {Gaussian differential privacy}.
\begin{definition}\label{def:gdp}
Let $\Phi$ denote the cumulative distribution function of the
standard normal distribution. 
For $\mu \geq 0$, let $G_\mu := T(\N(0,1), \N(\mu, 1)) \equiv \Phi(\Phi^{-1}(1-\alpha) - \mu)$.
An algorithm $\A$ is $\mu$-GDP if $T(\A(S), \A(S')) \geq G_\mu$ for any pair of neighboring datasets $S$ and $S'$.
\end{definition}

One advantage of $f$-differential privacy is that the composition of algorithms 
can be neatly handled. 
The composition primitive refers to an algorithm $\A$ that consists of $R$ algorithms $\A_1, \ldots, \A_R$, where $\A_i$ observes
both the input dataset and output from all previous algorithms. 
Let $f_1 = T(P_1, Q_1)$ and $f_2 = T(P_2, Q_2)$, 
\cite{dong2019gaussian} defines a binary
operator $\otimes$ on trade-off functions such that $f_1 \otimes f_2 = T(P_1 \times P_2,
Q_1 \times Q_2)$, where $P_1 \times P_2$ is the distribution product.  This operator
is commutative and associative, and provides elegant formulations for the composition
of private algorithms.
\begin{lemma}[\cite{dong2019gaussian}]
\label{lemma:fdp_composition}
If $\A_i$ is $f_i$-differentially private for $1 \le i \le R$, then the composed algorithm $\A$ is $f_1
\otimes \cdots \otimes f_R$-differentially private.  
\end{lemma}
\begin{lemma}[\cite{dong2019gaussian}]
\label{lemma:gdp_composition}
The $R$-fold composition of $\mu_i$-GDP algorithms is $\sqrt{\mu_1^2 + \cdots + \mu_n^2}$-GDP.
\end{lemma}

\subsection{Federated \texorpdfstring{$f$}{f}-Differential Privacy}
\label{sec:privacy_notion}
Section~\ref{sec:privacy_algo} has discussed that the privacy leakage of Client $j$ to Client $i$ is determined by the helper model $\hinoise$, which motivates the following definitions. 

Recall that $\Sj$ is the dataset of Client $j$.
Let $\Sjj$ denote a neighboring dataset of $\Sj$, 
i.e., $\Sjj$ and $\Sj$ differ by only one entry. 
Let $\bS = \left(S^{(1)}, \ldots, S^{(m)} \right)$ denote the joint dataset across clients. 
Let $M(\cdot) = (M_1(\cdot), \ldots, M_m(\cdot))$ be the randomized federated algorithm that returns the helper models to clients: $M_i(\bS) = \hinoise$ is the helper 
model for Client $i$. Note that for $M_i$, the usage of $\Sj_{j\neq i}$ is implicit: Client $i$ is blind to those datasets. We write $\bSj$ if it is neighboring with $\bS$ in the $j$-th component:
$ \bSj = (S^{(1)}, \ldots, \Sjj, \ldots, S^{(n)}). $
The following two definitions quantitatively describe how well every client could protect her/his own data against
the other clients.
\begin{definition}
\label{def:privacy_weak}
    A randomized federated learning algorithm $M$ satisfies the \emph{weak federated $f$-differential privacy} if for any $i \neq j$, it
holds that $T\big(M_i(\bS), M_i(\bSj)\big) \geq f.$
\end{definition}
\begin{definition}
\label{def:privacy_strong}
    Let $M_{-j}$ denote the randomized output of all the helper models except $j$.  $M$ satisfies
    the \emph{strong federated $f$-differential privacy} if it holds that for any
    $j$, 
    $T\big(M_{-j}(\bS), M_{-j}(\bSj)\big) \geq f.$
    This is equivalent to
    $
        T\big(\prod_{i\neq j}M_i(\bS), \prod_{i\neq j} M_i(\bSj)\big) \geq f.
    $
\end{definition}

We remark that Definition~\ref{def:privacy_weak} is a one-vs.-one privacy notion. Under this notation, every client is protected from the
attack from any other malicious client.
Definition~\ref{def:privacy_strong} is a one-vs.-all privacy notion. In the worst case, the other
clients would make allies and attack Client $i$ together. An algorithm $M$ satisfying
Definition~\ref{def:privacy_strong} could guarantee the privacy of Client $i$ even in this situation. In other words,
if $M$ satisfies the strong federated $f$-differential privacy, then it satisfies the weak federated $f$-differential privacy.

\subsection{Analysis}
\label{sec:privacy_analysis}
Let $\Htilde_i$ denote the update of $\hinoise$ on the server. 
In practice, if Client $i$ is not sampled for synchronization,
the algorithm does not release a model to Client $i$, thus the perfect privacy of all
the other clients' data is achieved. In the privacy analysis,
this is equivalent to releasing a constant number that carries zero information.
Letting $\bS_\Omega$ denote the subsampled dataset, we can write this update as: 
\begin{equation}
    \Htilde_i(\bS_\Omega) = \begin{cases}
    \hinoise(\bS_\Omega), & \text{if} \; i \in \Omega, \\
    \zero, & \text{otherwise}.
    \end{cases}
\end{equation}
Let $\sample$ denote the Possion subsampling of clients for synchronization.
The update of $\hinoise$ for one synchronization round is the subsampled algorithm $\Htilde_i \circ \sample$. 
We remark that
the subsampling step $\sample(S) = \Omega$ is an intermediate step that is not released, and the subsampled algorithm
$\Htilde_i \circ \sample$ should be considered as a whole. 
Our target to analyze is essentially the composition of $R$ copies of $\Htilde_i \circ \sample$:
\begin{equation}\label{eq:M_to_H}
\begin{aligned}
    & \hspace{2pt} T\big(M_i(\bS), M_i(\bSj)\big) \\
= & \hspace{2pt} T\big( (\Htilde_i \circ \sample)^{\otimes R}(\bS),  (\Htilde_i \circ \sample)^{\otimes R}(\bSj)\big).
\end{aligned}
\end{equation}

The analysis has three steps. We first need to understand the privacy guarantee of the algorithm $\Htilde_i$, without sampling. 
The second step is to figure out the guarantee of the subsampled algorithm $\Htilde_i \circ \sample$. Last, 
we apply the composition theorem of $f$-differential privacy to obtain the final guarantee. The results are presented in Lemma~\ref{lemma:one_update},
Lemma~\ref{lemma:one_update_sampled}, and Theorem~\ref{thm:alg_1} in order.

\begin{lemma}
\label{lemma:one_update}
For any Client $j$, suppose the local training of $\wjnoise$ is $f_j$-differentially private. It holds that
\[
T\big( \Htilde_i(\bS), \Htilde_i(\bSj) \big) \geq f_j, \; i \in [m].
\]
\end{lemma}
\begin{proof}
See Appendix~\ref{sec:proof_nosampling}.
\end{proof}
This lemma implies that for any Client $j$, the privacy guarantee holds uniformly the same for all the other clients.
Intuitively, the privacy loss is determined once Client $j$ dispatches $\wjnoise$,
and the subsequent post-processing of $\wjnoise$ will incur no extra privacy loss.
The privacy leakage to the other clients will only
differ if Client $j$ sends different models with different levels of noise to the other clients, explicitly or implicitly. Since each client only communicates with the server in \fedsync, we can guarantee the privacy protection
is uniform over all the other clients. 

Next, we analyze the subsampled algorithm $\Htilde_i \circ \sample$. 
Compared with the original algorithm, subsampling amplifies the privacy guarantee. Such amplification 
is due to the fact that if Client $j$ is not included in one round of synchronization,
it enjoys perfect privacy for that round. Our results are described formally in the following lemma.
\begin{lemma}
\label{lemma:one_update_sampled}
Let $g_{p, j} = \max(f_j, 1 - \alpha - p^2)$. Suppose the local training algorithm of $\wjnoise$ is $f_j$-differentially private. 
Consider the subsampled algorithm $\Htilde_i \circ \sample$ with $ 0 \leq p \leq 1$. For any $i \in [m]$, it holds
that
\[
 T\big( \Htilde_i \circ \sample(\bS), \Htilde_i \circ \sample(\bSj) \big) \geq g_{p, j}.
\]
\end{lemma}
\begin{proof}
See Appendix~\ref{sec:proof_sampled}. 
\end{proof}
We emphasize that the technical needs for analyzing the client sampling of \fedsync is different from the analysis of
private SGD with Poisson sampling \cite{bu2019deep}, and the existing results do not directly apply to our case. The main difference is that for \fedsync, the privacy loss of Client $j$ to Client $i$ is affected by whether $i$ and $j$ are both sampled in $\Omega$.
From the hypothesis testing point of view, the two distributions the adversary is trying to tell apart,
$\Htilde_i \circ \sample(\bS)$ and $\Htilde_i \circ \sample(\bSj)$, are both mixture models of two groups: one group contains the cases
both $i$ and $j$ are sampled, the other group contains the other cases. Whereas, for analyzing private SGD, only one of the two distributions need to be divided into two groups.

Finally, we apply the composition theorem of $f$-differential privacy (Lemma~\ref{lemma:fdp_composition}) to obtain the following results.
\begin{theorem}
\label{thm:alg_1}
Let $g_{p, j} = \max(f_j, 1 - \alpha - p^2)$ be defined as in Lemma~\ref{lemma:one_update_sampled}. It holds that 
\[
\begin{aligned}
T\big( M_i(\bS),  M_i(\bSj)\big) \geq g_{p, j}^{\otimes R}, \; i \in [m].
\end{aligned}
\]
Consequently, Algorithm~\ref{alg:fedsync} satisfies weak federated $f$-differential privacy for 
$\displaystyle f = g_{p, j_{\min}}^{\otimes R}$,
where $g_{p, j_{\min}} = \min \set{g_{p, 1}, \ldots, g_{p, m}}$.
It also satisfies strong federated $ \displaystyle g_{p, j_{\min}}^{\otimes (m-1)R}    $-differential privacy.
\end{theorem}
\begin{proof}
See Appendix~\ref{sec:proof_thm_alg_1}.
\end{proof}

\subsection{Local Private Training}
\begin{algorithm}[t]
\SetKw{kwInput}{Input:}
\SetKw{kwOutput}{Output:}
\SetEndCharOfAlgoLine{\relax}
\DontPrintSemicolon
\kwInput{loss $L$, dataset $\Sj$, helper model $\hjnoise$, batch size $B_j$, noise scale $\sigma_j$,
maximum gradient norm $C$, learning rates $\gamma_1, \ldots, \gamma_K$}\\
\textbf{Initialize:} $\wj \leftarrow \hjnoise$\;
\For{k = 1, \ldots K}{
    Sample $I \subseteq \set{1, \ldots, |\Sj|} $ with size $B_j$ uniformly at random \;
    \tcp{\small Compute and clip the per-sample gradient}
    \For{$\ell \in I$}{
        $v_\ell =  \nabla L(\wj, x_\ell, y_\ell)$\;
        $v_\ell \leftarrow v_\ell / \max\set{1, \norm{v_\ell}/C}$\;
    }
    $\displaystyle \wj \leftarrow \wj -  \frac{\gamma_k}{B} \bigg( \sum_{\ell \in I} v_\ell + \N(0, 4 C^2 \sigma_j^2 I) \bigg)$
}
\caption{Example Local Training of Client $j$ using NoisySGD}
\label{alg:localsgd}
\end{algorithm}

In this section, we present an example local training algorithm using noisy SGD as
the optimizer, see Algorithm~\ref{alg:localsgd}.
We analyze its privacy guarantee and present a final privacy bound of \fedsync
after injecting it into Algorithm~\ref{alg:fedsync}. 
Although Algorithm~\ref{alg:localsgd}
uses SGD as the optimizer, our results hold for a large number of other optimizers,
including Adam \cite{adam}, AdaGrad \cite{adagrad}, Momentum SGD \cite{momentum}, etc. 
In brief, this is because the statistics like
the momentum, the running mean of the gradient, are deterministic functions of the noisy gradient, thus
no additional privacy loss would be incurred for those computations.

Let $f_p = p f + (1 - p) \Id$ for some $p \in [0, 1]$.
Let $f^{-1}_p$ be the inverse function of $f_p$: $f^{-1}_p(\alpha) = \inf_{t \in [0, 1]} \set{f_p(t) \leq \alpha}$.
Define a trade-off function $C_p(f) = \min \set{f_p, f^{-1}_p}^{**}$ where  $f^{**}$ denotes
the double conjugate of $f$. The function $f_p$ is asymmetric in general but $C_p(f)$ is symmetric, see Figure~\ref{fig:cpf}.
\begin{theorem}\label{thm:alg_1_and_2}
Suppose Algorithm~\ref{alg:localsgd} is used in Algorithm~\ref{alg:fedsync} for the local private training. 
It holds that for any Client $j$,
\[
    T\big( M_i(\bS),  M_i(\bSj)\big) \geq  C_{B_j/n_j}(G_{1/\sigma_j})^{\otimes KR}, \; i \in [m].
\]
Furthermore, if $\frac{B_j}{n_j}\sqrt{KR} \rightarrow c_j$ as $\sqrt{KR} \rightarrow \infty$, then $C_{B_j/n_j}(G_{1/\sigma_j})^{\otimes KR} \rightarrow G_{\mu_j}$
where
\[
\mu_j = \sqrt{2} c_j \sqrt{e^{\sigma_j^{-2}} \Phi(1.5 \sigma_j^{-1}) + 3 \Phi(-0.5 \sigma_j^{-1}) - 2}.
\]
Consequently, Algorithm~\ref{alg:fedsync} satisfies weak federated $G_{\mu_{\max}}$-differential privacy
and strong federated $G_{\sqrt{m-1}\mu_{\max}}$-differential privacy,
where $\mu_{\max} = \max \set{\mu_1, \ldots, \mu_m}$.
\end{theorem}
\begin{proof}
See Appendix~\ref{sec:proof_thm_alg_1_and_2}. 
\end{proof}
\section{Experiments}
\label{sec:expr}
We use Algorithms~\ref{alg:fedsync} and \ref{alg:localsgd} to train private deep
learning models for two computer vision tasks: MNIST digit recoginition
\cite{lecun1998mnist} and CIFAR-10 object classification
\cite{krizhevsky2009learning}\footnote{Our code is available at
    \url{https://github.com/enosair/federated-fdp}.}.
To simulate the heterogeneous data distributions, we make non-IID partitions of the datasets, see below for the detailed descriptions. For all the experiments, we fix the aggregation parameter $\eta = 1$, use the interpolation method as in Equation~\eqref{eq:hi_intepolation} with $\alpha=0.1$ to compute the helper models, and clip the gradient with maximum norm $C=1$ when training private
models. For both tasks, we report the average testing accuracy along with the privacy guarantees we obtained, and compare with the non-private results under the same setting.  The algorithms and models we use might not yield the best possible prediction accuracy, 
but they are sufficient for the purposes of illustrating
our private notion and investigating the relative performance for private and non-private algorithms.

\subsection{Non-IID MNIST}
The MNIST dataset
contains 60,000 training images and 10,000 testing images.
We use a setup similar to \cite{mcmahan2016communication} to partition the data for 100 clients.
We sort the training data by digit label and evenly divide it into 400
shards. Each of 100 clients is assigned four random shards of
the data, so that most of the clients have examples of three or four digits.
For testing, each client will sample 200 examples with the same label she/he has seen
in training\footnote{In contrast to some previous works where only the training
data is non-IID, our testing data is also not identically distributed across the clients.}.
We use a CNN model with two convolution layers with 3 x 3 kernels,
followed by an FC layer with 128 units and
ReLu activation, and a final softmax output layer.
For local training, we use noisy Adam with base learning rate $0.001$.

\begin{table}[t]
\small
    \centering
    \begin{tabular}{c|c|c|c|c|c}
         \toprule
         $p$  & $\sigma$ & $R$ & $\mu_{\max}$ & Test Acc & Non-Pri Acc \\ \midrule \midrule
              & 1.0      & 93  & 2.71         & 90.03   &  98.74  \\
         1.0  & 0.9      & 83  & 3.10         & 90.25   &  98.72  \\
              & 0.75     & 64  & 3.96         & 90.10   &  98.55  \\
         \midrule
              & 1.0      & 194  & 3.92        & 90.02   &  98.90  \\
         0.5  & 0.9      & 176  & 4.51        & 90.02   &  98.83  \\
              & 0.75     & 127  & 5.58        & 90.11   &  98.54 \\
         \midrule
              & 1.0      & 386  & 5.52        & 90.00   &  98.75  \\
         0.25 & 0.9      & 325  & 6.13        & 90.00   &  98.75 \\
              & 0.75     & 245  & 7.75        & 90.04   &  98.55 \\
         \bottomrule
    \end{tabular}
    \caption{MNIST experiment results: the round of synchronization and corresponding privacy parameter
    when the average test accuracy reaches $90\%$.
    The privacy parameter is computed as in Theorem~\ref{thm:alg_1_and_2}.}
    \label{tbl:mnist}
\end{table}

\begin{table}[t]
    \centering
    \small
    \begin{tabular}{c|c|c|c|c|c|c}
        \toprule
         $\sigma$  & $B$  & $K$   & $R$ & Total Iter.    & Total Ex. & $\mu_{\max}$ \\ \midrule \midrule
                1  & 8    & 76    & 266 & 20216 & 161728 & 3.24\\
                   & 16   & 38    & 194 & 7372  & 117952 & 3.92\\
        \midrule
               0.9 & 8    & 76    & 229 & 17404 & 139232 & 3.64\\
                   & 16   & 38    & 176 & 6688  & 107008 & 4.51\\
        \midrule
              0.75 & 8    & 76    & 191 & 14516 & 116128 & 4.84\\
                   & 16   & 38    & 127 & 4826  & 77216  & 5.58\\
        \bottomrule
    \end{tabular}
    \caption{The trade-off between privacy and computation.
    The per-client total number of iterations is $KR$,
    the per-client total number of examples is $BKR$,
    and the results are reported when
    the average test accuracy reaches $90\%$.
    The privacy parameter $\mu_{\max}$ for small batch size ($B=8$) runs
    is roughly $0.83\times$ as large as obtained by the $B=16$ runs, yet the $B=8$ runs process
    approximately $1.39\times$ total number of data samples. The client sampling rate is $p=0.5$.
    }
    \label{tbl:mnist_batchsize}
 \end{table}

\begin{figure}[t]
    \centering
    \includegraphics[width=0.55\columnwidth]{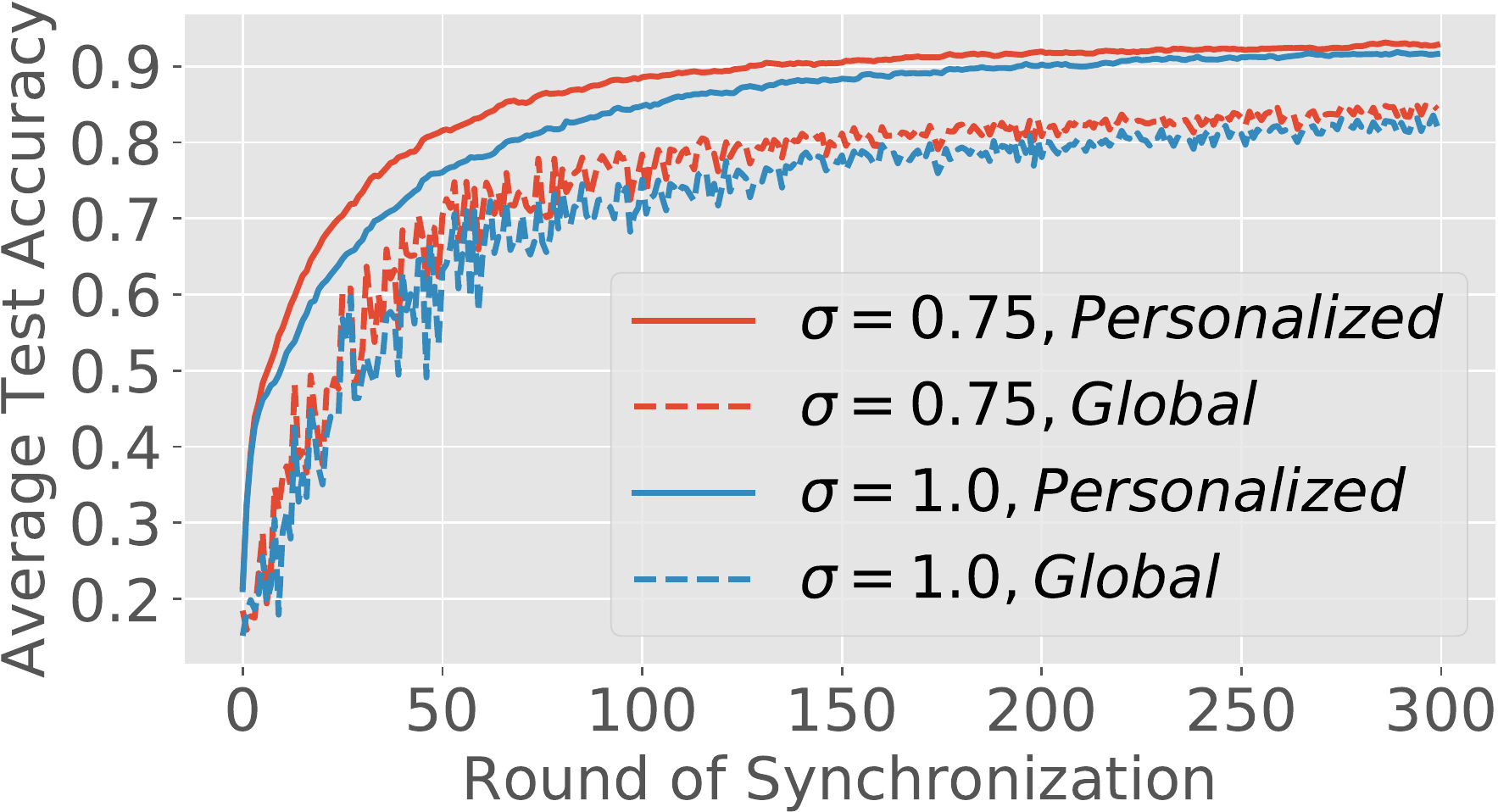}
    \vskip-5pt\caption{The personalized models outperform the global model in the MNIST experiments. The client sampling rate is $p=0.5$.}
    \label{fig:mnist_personalization}
    \vskip5pt
    \includegraphics[width=0.55\columnwidth]{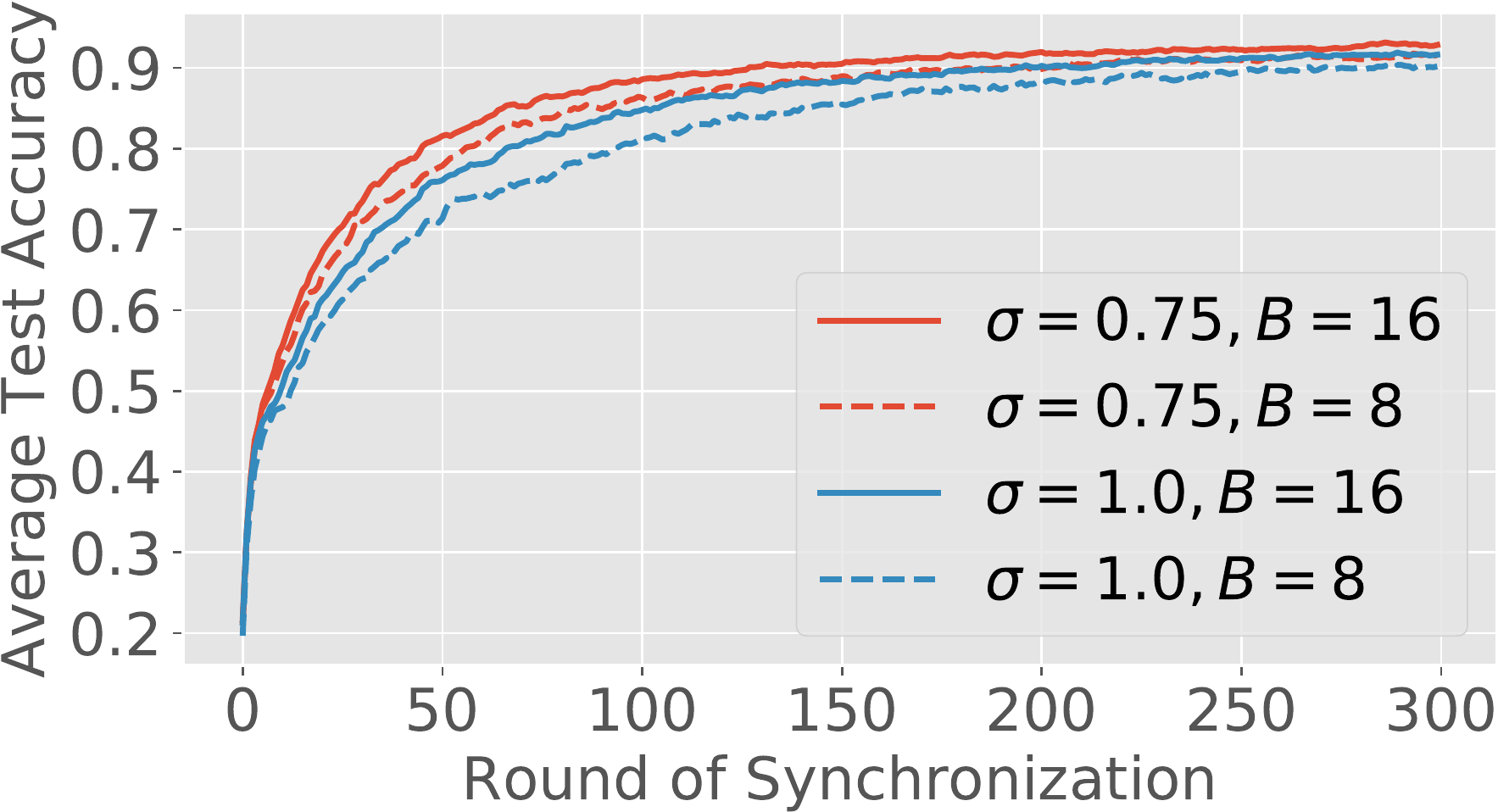}
    \vskip-5pt
    \caption{A larger batch size leads to faster convergence for the MNIST experiments. The client sampling rate is $p=0.5$.}
    \label{fig:mnist_batch_size_comp}
\end{figure}

\paragraph{Performance of Private Models.} We test three values for the client sampling rate $p$: $0.25, 0.5$ and $1.0$, and three values of noise scale
$\sigma: 0.75, 0.9$, and $1.0$. We use batch size $B=16$ and run local training for $K=38$ iterations between synchronization. The total number of samples processed between synchronization is $608$, so we are approximately running local training for one epoch.
Table~\ref{tbl:mnist} reports the synchronization rounds $R$ and the privacy parameter $\mu_{\max}$,
when the average prediction accuracy across 100 clients is above $90\%$.
Table~\ref{tbl:mnist} also shows an intuitive phenomenon:
a larger client sampling rate and a smaller noise level lead to faster convergence, see also Figure~\ref{fig:mnist_sampling_rate}.
One might notice the privacy parameter is slighter larger than one usually see in
a centralized training setting.  Recall that Theorem~\ref{thm:alg_1_and_2} states that the privacy parameter $\mu_j$ scales linearly with
a constant $c_j$.
Loosely speaking, $c_j$ is the product of
the data sampling rate $\frac{B_j}{n_j}$ and the squared training iterations $\sqrt{KR}$:
$\frac{B_j}{n_j}\sqrt{KR} \rightarrow c_j$ as $KR \rightarrow \infty$.
In our simulated federated environment, each client holds only $1\%$ data of
the whole dataset, and the batch size is approximately $1/16$ of the normal setting. Therefore,
the effective data sampling rate $B/n$ is much larger. Besides, it also takes much more
iterations for the algorithm to converge in the federated setting. This leads to a
larger privacy parameter. Interesting, there is also a trade-off between data sampling
rate and computing iterations which might affect the privacy parameter, and we shall
discuss this later in this section.


\paragraph{Accuracy Gain from Personalization.}
Figure~\ref{fig:mnist_personalization} investigates the personalization performance of our approach. It compares the
average test accuracy for the private global model and private personalized models.
We plot the results when the noise scale $\sigma=0.75$ and $1.0$, where the client sampling rate is $p=0.5$. For both cases, personalized models significantly outperform the global model. The results for the other sampling rates are similar.

\paragraph{Privacy vss Computation Trade-off.}
As presented in Theorem~\ref{thm:alg_1_and_2}, the privacy parameter $\mu_j$ scales linearly with
a constant $c_j$. Informally, this is the product of the data sampling rate $\frac{B_j}{n_j}$ and the squared training iterations
$\sqrt{KR}$. Since $c_j$ scales linearly with $B_j$ but sublinearly with $K$, using a smaller batch size would lead to a more private model if one processes the same amount
of total data examples. For example,
$\frac{B_j/2}{n_j}\sqrt{2KR} < \frac{B_j}{n_j}\sqrt{KR}$.
However, 
the batch size has great impact on the rate of convergence. Figure~\ref{fig:mnist_batch_size_comp}
illustrates this phenomenon. Fixing the client sampling rate $p=0.5$, we decrease the batch size from $B=16$ to $B=8$,
and double the number of local iterations to $K=76$.
The small batch size ($B=8$) runs take more rounds to achieve the same test accuracy, which means it
processes more data examples in total.
Table~\ref{tbl:mnist_batchsize} demonstrates such a trade-off between privacy and computation.
We compare the total number of training
iterations $KR$, the per-client total training examples $BKR$, and the privacy parameter $\mu_{\max}$.
As before, the results are reported when the average test accuracy achieves $90\%$.
Compared with the large batch size ($B=16$) runs, the small batch size ($B=8$) runs
obtain smaller privacy parameters, which are roughly $0.83\times$ as obtained by the $B=16$ runs.
However, they take approximately $2.78\times$ iterations to achieve $90\%$ accuracy,
which translates to $1.39\times$ total number of samples.

\subsection{Non-IID CIFAR}
The CIFAR-10 dataset contains 50,000 training images and 10,000 test images
of 10 classes. We use the same experiment setup as \cite{hsu2019measuring}.
There are 100 clients, each holds 500 training images and 200 testing images.
For each client, we generate data using the following probabilistic model: \\
   \hspace*{1em}1. Sample the class probability $q \sim \text{Dir}(\beta)$.\\
   \hspace*{1em}2. Sample $\theta^\text{tr} \sim \text{Multinomial}(q, 500)$.\\
   \hspace*{1em}3. Sample $\theta^\text{tr}_i$ images with label $i$ from the training set without replacement. \\
   \hspace*{1em}4. Likewise, sample $\theta^\text{te} \sim \text{Multinomial}(q, 200)$ and the testing data accordingly.\\
The hyperparameter $\alpha$ controls the heterogeneity of the client data distributions.  With $\beta \rightarrow \infty$,
all the clients have identical class distributions; with $\beta \rightarrow 0$, the probability vector $q$ will
be one-hot and each client will hold samples from only one class. We use $\beta=0.5$ throughout our experiments,
see Figure~\ref{fig:cifar_class_dict} for a visual illustration of the label proportions.
Due to the GPU memory limit, we use the CNN model from the TensorFlow tutorial\footnote{\url{https://www.tensorflow.org/tutorials/images/cnn}.}, like the previous work \cite{mcmahan2016communication, hsu2019measuring}. This architecture is not state-of-the-art for CIFAR,
but sufficient to demonstrate the relative performance for private and non-private models.

\begin{figure}[t]
    \centering
    \includegraphics[width=0.6\columnwidth]{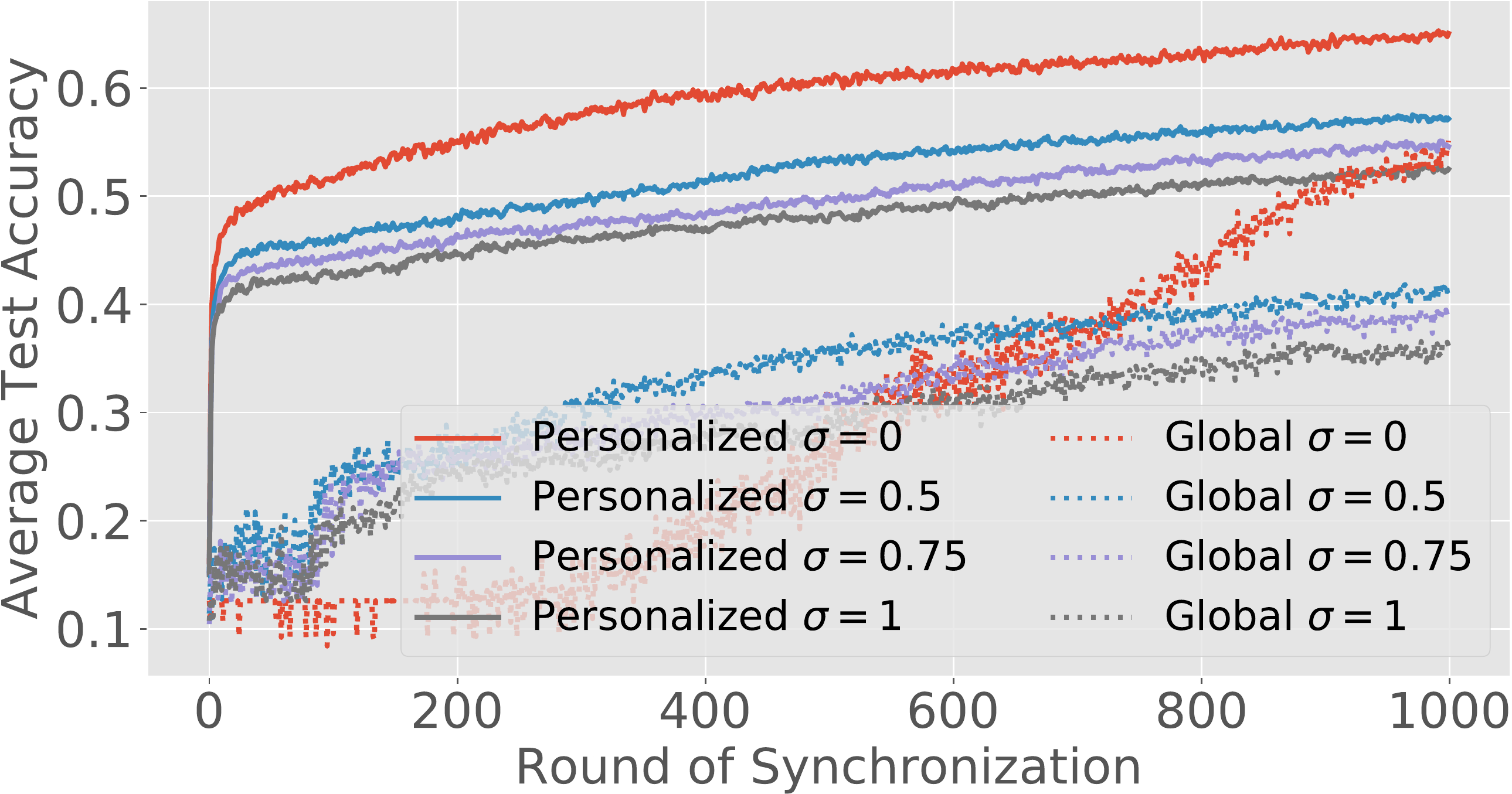}
    \vskip-5pt\caption{Average top-1 test accuracy vs. synchronization rounds for the CIFAR-10 experiments. The client sampling rate
    is $p=0.5$.}
    \label{fig:cifar_cnn}
\end{figure}

\begin{table}[tbh]
\small
    \centering
    \begin{tabular}{c|c|c|c|c|c}
         \toprule
         $p$  & $\sigma$ & $R$ & $\mu_{\max}$ & Top-1 Acc & Non-Pri Acc \\ \midrule \midrule
              & 1.0      & 468  & 6.70        & 52.03   &  64.72  \\
         1.0  & 0.75     & 321  & 9.77        & 52.22   &  62.55  \\
              & 0.5      & 207  & 26.81       & 52.23   &  59.61  \\
         \midrule
              & 1.0      & 904  & 9.31        & 52.07   &  64.65  \\
         0.5  & 0.75     & 671  & 14.13       & 52.04   &  62.53  \\
              & 0.5      & 405  & 37.51       & 52.01   &  59.85  \\
         \bottomrule
    \end{tabular}
    \caption{CIFAR-10 experiment results: the round of synchronization and the corresponding privacy parameter
    when the average top-1 accuracy reaches $52\%$.}
    \label{tbl:cifar_cnn}
\end{table}

We observe that Adam is more stable than SGD for training private models, although
SGD generalizes better on non-private models. Thereby, we train the models by
noisy Adam with base learning $0.005$ and weight decay $0.0005$. The learning rate is decayed
by a factor of $0.99$ every $10$ epochs. We use batch size $B=16$ and run local training for $K=32$ iterations.
Figure~\ref{fig:cifar_cnn} plots the top-1 test accuracy curve when the client samping rate $p=0.5$.
We can observe the privacy-accuracy trade-off: the test accuracy moderately decreases as the model becomes more private, i.e. trained with larger $\sigma$. Meanwhile, personalized models still outperform the global model.
Table~\ref{tbl:cifar_cnn} reports the round of synchronization and corresponding privacy parameter
when the average top-1 accuracy reaches $52\%$.



\subsubsection*{Acknowledgments}
We are grateful to Arun Kuchibhotla and Jinshuo Dong for insightful discussions.
This work was supported in part by NIH through R01-GM124111 and RF1-AG063481, NSF through CAREER DMS-1847415, CCF-1763314, and CCF-1934876, a Facebook Faculty Research Award, and an Alfred
Sloan Research Fellowship.

{
\small
\bibliographystyle{alpha}
\bibliography{ref}

\newcommand{\etalchar}[1]{$^{#1}$}
\begin{thebibliography}{MMR{\etalchar{+}}17}

\bibitem[ACG{\etalchar{+}}16]{abadi2016deep}
Martin Abadi, Andy Chu, Ian Goodfellow, H~Brendan McMahan, Ilya Mironov, Kunal
  Talwar, and Li~Zhang.
\newblock Deep learning with differential privacy.
\newblock In {\em Proceedings of the 2016 ACM SIGSAC Conference on Computer and
  Communications Security}, pages 308--318, 2016.

\bibitem[BBG{\etalchar{+}}19]{balle2019hypothesis}
Borja Balle, Gilles Barthe, Marco Gaboardi, Justin Hsu, and Tetsuya Sato.
\newblock Hypothesis testing interpretations and renyi differential privacy.
\newblock {\em arXiv preprint arXiv:1905.09982}, 2019.

\bibitem[BDLS20]{bu2019deep}
Zhiqi Bu, Jinshuo Dong, Qi~Long, and Weijie~J Su.
\newblock Deep learning with {G}aussian differential privacy.
\newblock {\em Harvard Data Science Review}, 2020(23), 2020.

\bibitem[BDRS18]{tcdp}
Mark Bun, Cynthia Dwork, Guy~N Rothblum, and Thomas Steinke.
\newblock Composable and versatile privacy via truncated cdp.
\newblock In {\em Proceedings of the 50th Annual ACM SIGACT Symposium on Theory
  of Computing}, pages 74--86, 2018.

\bibitem[BS16]{concentrated2}
Mark Bun and Thomas Steinke.
\newblock Concentrated differential privacy: Simplifications, extensions, and
  lower bounds.
\newblock In {\em Theory of Cryptography Conference}, pages 635--658. Springer,
  2016.

\bibitem[DHS11]{adagrad}
John Duchi, Elad Hazan, and Yoram Singer.
\newblock Adaptive subgradient methods for online learning and stochastic
  optimization.
\newblock {\em Journal of machine learning research}, 12(7), 2011.

\bibitem[DKM{\etalchar{+}}06]{Dwork2006}
Cynthia Dwork, Krishnaram Kenthapadi, Frank McSherry, Ilya Mironov, and Moni
  Naor.
\newblock Our data, ourselves: Privacy via distributed noise generation.
\newblock In {\em Proceedings of the 24th Annual International Conference on
  The Theory and Applications of Cryptographic Techniques}, EUROCRYPT'06, pages
  486--503, Berlin, Heidelberg, 2006. Springer-Verlag.

\bibitem[DKM20]{deng2020adaptive}
Yuyang Deng, Mohammad~Mahdi Kamani, and Mehrdad Mahdavi.
\newblock Adaptive personalized federated learning.
\newblock {\em arXiv preprint arXiv:2003.13461}, 2020.

\bibitem[DMNS06]{Dwork2006_2}
Cynthia Dwork, Frank McSherry, Kobbi Nissim, and Adam Smith.
\newblock Calibrating noise to sensitivity in private data analysis.
\newblock In {\em Proceedings of the Third Conference on Theory of
  Cryptography}, TCC'06, pages 265--284, Berlin, Heidelberg, 2006.
  Springer-Verlag.

\bibitem[DR14]{dwork2014algorithmic}
Cynthia Dwork and Aaron Roth.
\newblock The algorithmic foundations of differential privacy.
\newblock {\em Foundations and Trends in Theoretical Computer Science},
  9(3-4):211--407, 2014.

\bibitem[DR16]{concentrated}
Cynthia Dwork and Guy~N Rothblum.
\newblock Concentrated differential privacy.
\newblock {\em arXiv preprint arXiv:1603.01887}, 2016.

\bibitem[DRS19]{dong2019gaussian}
Jinshuo Dong, Aaron Roth, and Weijie~J Su.
\newblock Gaussian differential privacy.
\newblock {\em To appear in Journal of the Royal Statistical Society: Series B
  (Statistical Methodology)}, 2019.

\bibitem[DTN20]{dinh2020moreau}
Canh~T Dinh, Nguyen~H Tran, and Tuan~Dung Nguyen.
\newblock Personalized federated learning with moreau envelopes.
\newblock {\em arXiv preprint arXiv:2006.08848}, 2020.

\bibitem[EGS03]{evfimievski2003limiting}
Alexandre Evfimievski, Johannes Gehrke, and Ramakrishnan Srikant.
\newblock Limiting privacy breaches in privacy preserving data mining.
\newblock In {\em Proceedings of the twenty-second ACM SIGMOD-SIGACT-SIGART
  symposium on Principles of database systems}, pages 211--222, 2003.

\bibitem[FJR15]{fredrikson2015model}
Matt Fredrikson, Somesh Jha, and Thomas Ristenpart.
\newblock Model inversion attacks that exploit confidence information and basic
  countermeasures.
\newblock In {\em Proceedings of the 22nd ACM SIGSAC Conference on Computer and
  Communications Security}, pages 1322--1333, 2015.

\bibitem[GKN17]{geyer2017differentially}
Robin~C Geyer, Tassilo Klein, and Moin Nabi.
\newblock Differentially private federated learning: A client level
  perspective.
\newblock {\em arXiv preprint arXiv:1712.07557}, 2017.

\bibitem[HCZ{\etalchar{+}}20]{huang2020Attentive}
Yutao Huang, Lingyang Chu, Zirui Zhou, Lanjun Wang, Jiangchuan Liu, Jian Pei,
  and Yong Zhang.
\newblock Personalized federated learning: An attentive collaboration approach.
\newblock {\em arXiv preprint arXiv:2007.03797}, 2020.

\bibitem[HHR{\etalchar{+}}16]{hsu2016private}
Justin Hsu, Zhiyi Huang, Aaron Roth, Tim Roughgarden, and Zhiwei~Steven Wu.
\newblock Private matchings and allocations.
\newblock {\em SIAM Journal on Computing}, 45(6):1953--1984, 2016.

\bibitem[HHRW16]{hsu2016jointly}
Justin Hsu, Zhiyi Huang, Aaron Roth, and Zhiwei~Steven Wu.
\newblock Jointly private convex programming.
\newblock In {\em Proceedings of the twenty-seventh annual ACM-SIAM symposium
  on Discrete algorithms}, pages 580--599. SIAM, 2016.

\bibitem[HQB19]{hsu2019measuring}
Tzu-Ming~Harry Hsu, Hang Qi, and Matthew Brown.
\newblock Measuring the effects of non-identical data distribution for
  federated visual classification.
\newblock {\em arXiv preprint arXiv:1909.06335}, 2019.

\bibitem[HR20]{hanzely2020mixture}
Filip Hanzely and Peter Richt{\'a}rik.
\newblock Federated learning of a mixture of global and local models.
\newblock {\em arXiv preprint arXiv:2002.05516}, 2020.

\bibitem[KB14]{adam}
Diederik~P Kingma and Jimmy Ba.
\newblock Adam: A method for stochastic optimization.
\newblock {\em arXiv preprint arXiv:1412.6980}, 2014.

\bibitem[KH09]{krizhevsky2009learning}
Alex Krizhevsky and Geoffrey Hinton.
\newblock Learning multiple layers of features from tiny images.
\newblock 2009.

\bibitem[KLN{\etalchar{+}}11]{localdp}
Shiva~Prasad Kasiviswanathan, Homin~K Lee, Kobbi Nissim, Sofya Raskhodnikova,
  and Adam Smith.
\newblock What can we learn privately?
\newblock {\em SIAM Journal on Computing}, 40(3):793--826, 2011.

\bibitem[KMY{\etalchar{+}}16]{konevcny2016federated}
Jakub Kone{\v{c}}n{\`y}, H~Brendan McMahan, Felix~X Yu, Peter Richt{\'a}rik,
  Ananda~Theertha Suresh, and Dave Bacon.
\newblock Federated learning: Strategies for improving communication
  efficiency.
\newblock {\em arXiv preprint arXiv:1610.05492}, 2016.

\bibitem[KOV17]{kairouz2017composition}
Peter Kairouz, Sewoong Oh, and Pramod Viswanath.
\newblock The composition theorem for differential privacy.
\newblock {\em IEEE Transactions on Information Theory}, 63(6):4037--4049,
  2017.

\bibitem[KPRU14]{kearns2014mechanism}
Michael Kearns, Mallesh Pai, Aaron Roth, and Jonathan Ullman.
\newblock Mechanism design in large games: Incentives and privacy.
\newblock In {\em Proceedings of the 5th conference on Innovations in
  theoretical computer science}, pages 403--410, 2014.

\bibitem[LeC98]{lecun1998mnist}
Yann LeCun.
\newblock The mnist database of handwritten digits.
\newblock 1998.

\bibitem[LHC{\etalchar{+}}19]{liu2019investigating}
Changchang Liu, Xi~He, Thee Chanyaswad, Shiqiang Wang, and Prateek Mittal.
\newblock Investigating statistical privacy frameworks from the perspective of
  hypothesis testing.
\newblock {\em Proceedings on Privacy Enhancing Technologies},
  2019(3):233--254, 2019.

\bibitem[LKCT19]{li2019privatemeta}
Jeffrey Li, Mikhail Khodak, Sebastian Caldas, and Ameet Talwalkar.
\newblock Differentially private meta-learning.
\newblock {\em arXiv preprint arXiv:1909.05830}, 2019.

\bibitem[LSZ{\etalchar{+}}18]{li2018fedprox}
Tian Li, Anit~Kumar Sahu, Manzil Zaheer, Maziar Sanjabi, Ameet Talwalkar, and
  Virginia Smith.
\newblock Federated optimization in heterogeneous networks.
\newblock {\em arXiv preprint arXiv:1812.06127}, 2018.

\bibitem[Mir17]{mironov2017renyi}
Ilya Mironov.
\newblock R{\'e}nyi differential privacy.
\newblock In {\em 2017 IEEE 30th Computer Security Foundations Symposium
  (CSF)}, pages 263--275. IEEE, 2017.

\bibitem[MMR{\etalchar{+}}17]{mcmahan2016communication}
H~Brendan McMahan, Eider Moore, Daniel Ramage, Seth Hampson, and Blaise
  Aguera~y Arcas.
\newblock Communication-efficient learning of deep networks from decentralized
  data.
\newblock In {\em AISTATS}, 2017.

\bibitem[MRTZ18]{mcmahan2017learning}
H~Brendan McMahan, Daniel Ramage, Kunal Talwar, and Li~Zhang.
\newblock Learning differentially private recurrent language models.
\newblock In {\em ICLR}, 2018.

\bibitem[MV16]{complexity}
Jack Murtagh and Salil Vadhan.
\newblock The complexity of computing the optimal composition of differential
  privacy.
\newblock In {\em Theory of Cryptography Conference}, pages 157--175. Springer,
  2016.

\bibitem[Qia99]{momentum}
Ning Qian.
\newblock On the momentum term in gradient descent learning algorithms.
\newblock {\em Neural networks}, 12(1):145--151, 1999.

\bibitem[SSSS17]{shokri2017membership}
Reza Shokri, Marco Stronati, Congzheng Song, and Vitaly Shmatikov.
\newblock Membership inference attacks against machine learning models.
\newblock In {\em 2017 IEEE Symposium on Security and Privacy (SP)}, pages
  3--18. IEEE, 2017.

\bibitem[SZY{\etalchar{+}}20]{singh2020differentially}
Ishika Singh, Haoyi Zhou, Kunlin Yang, Meng Ding, Bill Lin, and Pengtao Xie.
\newblock Differentially-private federated neural architecture search.
\newblock {\em arXiv preprint arXiv:2006.10559}, 2020.

\bibitem[TBA{\etalchar{+}}19]{truex2019hybrid}
Stacey Truex, Nathalie Baracaldo, Ali Anwar, Thomas Steinke, Heiko Ludwig, Rui
  Zhang, and Yi~Zhou.
\newblock A hybrid approach to privacy-preserving federated learning.
\newblock In {\em Proceedings of the 12th ACM Workshop on Artificial
  Intelligence and Security}, pages 1--11, 2019.

\bibitem[WZ10]{wasserman2010statistical}
Larry Wasserman and Shuheng Zhou.
\newblock A statistical framework for differential privacy.
\newblock {\em Journal of the American Statistical Association},
  105(489):375--389, 2010.

\bibitem[ZDLS20]{zheng2020sharp}
Qinqing Zheng, Jinshuo Dong, Qi~Long, and Weijie~J Su.
\newblock Sharp composition bounds for {G}aussian differential privacy via
  {E}dgeworth expansion.
\newblock In {\em International Conference on Machine Learning}, pages
  11420--11435, 2020.

\end{thebibliography}
}

\clearpage
\appendix
\section{Proof of Lemma~\ref{lemma:one_update}}
\label{sec:proof_nosampling}
As discussed in Section~\ref{sec:privacy_algo}, the privacy loss occurred at $\hinoise$ is lower bounded by the privacy loss occurred at $\wgnoise$. To keep our analysis general for all algorithms that fit in \fedsync, we shall assume no knowledge of $F_i$ and analyze $\wgnoise$.

Without sampling, we can write the update of $\wgnoise$ as:
\begin{equation}
    (1 - \eta) \wgnoise + \frac{ \eta }{m} \sum_{i\in [m]} \winoise(\Si).
\end{equation}
This is fully invertible function of $\wjnoise$, so that the privacy loss is of the updated $\wgnoise$
would be the same $\wjnoise$, i.e.,
\[ T\bigg( \Htilde_i(\bS), \Htilde_i(\bSj) \bigg) = T\bigg( \wjnoise(\Sj), \wjnoise(\Sjj) \bigg). \]
This lemma thus follows by the assumption.

\section{Proof of Lemma~\ref{lemma:one_update_sampled}.}
\label{sec:proof_sampled}
Let $\omega \in [0,1]^{m}$ be the indicator vector of the Possion sampling outcome: $\omega_i = 1$ if Client $i$ is selected in synchronization, i.e. $i \in \Omega$. We use $p_\omega$ to denote the probability that $\omega$ appears, namely,
$p_\omega = p^{s}(1-p)^{m-s}$ if $\omega$ has $s$ nonzero entries.

Let $E = \set{\omega:\,\omega_i = \omega_j = 1}$ denote that event that both Client $i$ and $j$ are selected, and let $E^c$ denote the complementary event that not both of them are selected.
The output distribution of the subsampled algorithm $\Htilde \circ \sample$ on dataset $S$ can be written as a mixture model
\begin{equation}
    \Htilde_i \circ \sample(\bS) = \sum_{\omega \in E} p_w P_\omega + \sum_{\omega \in E^c}  p_w Q_\omega,
\end{equation}
where we use $P_\omega$ to denote the output distribution associated with $\omega$ if $\omega \in E$, and use $Q_\omega$ for the other case.
It is easy to see that $P_\omega$ depends on dataset $\Sj$ but $Q_\omega$ does not. With the neighboring dataset $\bSj$, the distribution $\Htilde_i \circ \sample(S')$ can also be written as a mixture, yet only the components corresponding to cases where both $i$ and $j$ are selected will change. Specifically,
\begin{equation}
    \Htilde_i \circ \sample(\bSj) = \sum_{\omega \in E} p_w P'_\omega + \sum_{\omega \in E^c}  p_w Q_\omega.
\end{equation}
The following technical lemma helps us bound the trade-off function between $\Htilde_i \circ \sample(\bS)$ and $\Htilde_i \circ \sample(\bSj)$.

\begin{lemma}\label{lemma:tradeoff_two_mixtures}
Let $\F$ be an event space and $\F = E \cup E^c$ is a valid partition of $\F$.
Let $\omega$ denote an arbitrary event in $\F$, whose probability is $p_w$.
We have $\sum_{w \in\F} p_w = 1$. For each event $\omega \in \F$, $P_w$, $P'_w$ and $Q_w$ are distributions reside on a common
sample space. Consider two mixture distributions $\displaystyle A = \sum_{\omega \in E} p_\omega P_\omega + \sum_{\omega \in E^c} p_\omega Q_\omega$ and $\displaystyle B = \sum_{\omega \in E} p_\omega P'_\omega + \sum_{\omega \in E^c} p_\omega Q_\omega$. If there exists a trade-off function $f$ such that $T(P_\omega, P'_\omega) \geq f$ for all $\omega$, it holds that
\[
T(A, B)(\alpha) \geq \max \set{f(\alpha), 1 - \alpha - p_E}.
\]
\end{lemma}

Under the context of our problem, it holds that $\P(E) = p^2$ and $\P(E^c) = 1 - p^2$ due to the independence of sampling Client $i$ and $j$. Besides, for any fixed $\omega \in E$, using the same argument for Lemma~\ref{lemma:one_update},
we have $T(P_\omega, P'_\omega) =
T\big( \Htilde_i(\bS_\Omega), \Htilde_i(\bSj_\Omega) \big) \geq  f_j$. This proofs our results.

\subsection{Proof of Lemma~\ref{lemma:tradeoff_two_mixtures}}
\begin{proof}
Let $p_E = \P(w \in E)$. We can write
$$A = p_E \sum_{\omega \in E} p_{\omega | E} P_\omega + (1 - p_E) \sum_{\omega \in E^c} p_{\omega | E^c} Q_\omega$$
and
$$B = p_E \sum_{\omega \in E} p_{\omega | E} P'_\omega + (1 - p_E) \sum_{\omega \in E^c} p_{\omega | E^2} Q_\omega.$$

Suppose a rejection rule $\phi$ achieves type I error $\alpha$:
\begin{equation}\label{eq:lemma_mixutre_alpha}
    \alpha = \E_A[\phi] = p_E \sum_{\omega \in E} p_{\omega | E} \E_{P_\omega}[\phi] + (1 - p_E) \sum_{\omega \in E^c} p_{\omega | E^c} \E_{Q_\omega}[\phi].
\end{equation}

The type II error of $\phi$ is
\begin{equation}
\begin{aligned}
1 - \E_B[\phi]
& = 1 - p_E \sum_{\omega \in E} p_{\omega | E} \E_{P'_\omega}[\phi] - (1 - p_E) \sum_{\omega \in E^c} p_{\omega | E^c} \E_{Q_\omega}[\phi] \\
& = 1 - p_E + p_E\left( 1 - \sum_{\omega \in E} p_{\omega | E} \E_{P'_\omega}[\phi]    \right) - (1- p_E) \sum_{\omega \in E^c} p_{\omega | E^c} \E_{Q_\omega}[\phi] \\
& = p_E\left( 1 - \sum_{\omega \in E} p_{\omega | E} \E_{P'_\omega}[\phi]    \right) + (1 - p_E)\left( 1- \sum_{\omega \in E^c} p_{\omega | E^c} \E_{Q_\omega}[\phi]  \right)  \\
& =  p_E\left( \sum_{\omega \in E} p_{\omega | E} \left(1 - \E_{P'_\omega}[\phi]  \right)  \right) + (1 - p_E)\left( 1- \sum_{\omega \in E^c} p_{\omega | E^c} \E_{Q_\omega}[\phi]  \right) \\
& \stackrel{(i)}{\geq} p_E\left( \sum_{\omega \in E} p_{\omega | E} f ( \E_{P_\omega}[\phi] )   \right) + (1 - p_E)\left( 1- \sum_{\omega \in E^c} p_{\omega | E^c} \E_{Q_\omega}[\phi]  \right) \\
& \stackrel{(ii)}{\geq} p_E\left( \sum_{\omega \in E} p_{\omega | E} f( \E_{P_\omega}[\phi] )   \right) + (1 - p_E)f\left (\sum_{\omega \in E^c} p_{\omega | E^c} \E_{Q_\omega}[\phi]  \right) \\
& \stackrel{(iii)}{\geq} p_E f \left( \sum_{\omega \in E} p_{\omega | E} \E_{P_\omega}[\phi]   \right) + (1 - p_E)f\left (\sum_{\omega \in E^c} p_{\omega | E^c} \E_{Q_\omega}[\phi]  \right) \\
& \stackrel{(iv)}{\geq} f\left( p_E \sum_{\omega \in E} p_{\omega | E} \E_{P_\omega}[\phi] +  (1-p_E) \sum_{\omega \in E^c} p_{\omega | E^c} \E_{Q_\omega}[\phi]   \right) \\
& = f(\alpha),
\end{aligned}
\end{equation}
where
\begin{enumerate}
    \item[] (i) follows from the definition of the trade-off function: $T(P_\omega, P'_\omega) \geq f $ implies $1 - \E_{P'_\omega}[\phi] \geq f(\E_{P_\omega}[\phi])$,
    \item[] (ii) follows from the property of trade-off functions: $f(\alpha) \leq 1 - \alpha, \forall \alpha \in [0, 1]$,
    \item[] (iii) and (iv) follows from the Jensen's inequality for convex functions ($f$ is convex).
\end{enumerate}
It also holds that
\begin{equation}
\begin{aligned}
1 - \E_B[\phi]
& = 1 - p_E \sum_{\omega \in E} p_{\omega | E} \E_{P'_\omega}[\phi] - (1 - p_E) \sum_{\omega \in E^c} p_{\omega | E^c} \E_{Q_\omega}[\phi] \\
& \stackrel{(v)}{=} 1 - p_E \sum_{\omega \in E} p_{\omega | E} \E_{P'_\omega}[\phi] - \set{ \alpha -p_E \sum_{\omega \in E} p_{\omega | E} \E_{P_\omega}[\phi] } \\
& = 1 - \alpha - p_E + p_E \set{
1 - \sum_{\omega \in E} p_{\omega | E} \E_{P'_\omega}[\phi] + \sum_{\omega \in E} p_{\omega | E} \E_{P_\omega}[\phi]
} \\
& = 1 - \alpha - p_E + p_E \set{
\sum_{\omega \in E} p_{\omega | E} \left( 1 - \E_{P'_\omega}[\phi] + \E_{P_\omega}[\phi] \right)
} \\
& \stackrel{(vi)}{\geq} 1 - \alpha - p_E + p_E\set{\sum_{\omega \in E} p_{\omega | E}\big(1 - \text{TV}(P'_\omega, P_\omega)\big)} \\
& \stackrel{(vii)}{\geq} 1 - \alpha - p_E.
\end{aligned}
\end{equation}
The equality (v) follows from Equation~\eqref{eq:lemma_mixutre_alpha}.
For (vi) and (vii), consider the rejection rule $\phi$ for testing
$P_\omega$ versus $P'_\omega$. The type I error is $\alpha_\omega = \E_{P_\omega}[\phi]$ and type II error is $\beta_\omega = 1 - \E_{P'_\omega}[\phi]$. It is well known that
\[ \alpha_\omega + \beta_\omega \geq 1 - \text{TV}(P_\omega, P'_\omega), \]
where $\text{TV}(P_\omega, P'_\omega)$ is the total variation distance between $P_\omega$ and $P'_\omega$, which takes value between $0$ and $1$.
\end{proof}

\section{Proof of Theorem~\ref{thm:alg_1}}
\label{sec:proof_thm_alg_1}

Lemma~\ref{lemma:one_update_sampled}
shows that for any $i \in [m]$,
\begin{equation}
    T\big( \Htilde_i \circ \sample(\bS), \Htilde_i \circ \sample(\bSj) \big) \geq g_{p, j}.
\end{equation}
Recall that Equation~\ref{eq:M_to_H} established the equivalence between $T\big( (\Htilde_i \circ \sample)^{\otimes R}(\bS),  (\Htilde_i \circ \sample)^{\otimes R}(\bSj)\big)$ and $T\big( M_i(\bS),  M_i(\bSj)\big)$. By the composition theorem of $f$-differential privacy (Lemma~\ref{lemma:fdp_composition}), we have that for any $i \in [m]$,
\begin{equation}\label{eq:H_fdp_comp}
\begin{aligned}
& && T\big( M_i(\bS),  M_i(\bSj)\big)
& = && T\big( (\Htilde_i \circ \sample)^{\otimes R}(\bS),  (\Htilde_i \circ \sample)^{\otimes R}(\bSj)\big)
& \geq && g_{p, j}^{\otimes R}.
\end{aligned}
\end{equation}
The above result holds for a fixed Client $j$. Since the weak federated $f$-differential privacy notion (Definition~\ref{def:privacy_weak}) is defined for any pairs of $i, j$ such that $i \neq j$, we need to take the ``least private'' trade-off function as our lower bound. That is
$\displaystyle g_{p, j_{\min}}^{\otimes R}$,
where $g_{p, j_{\min}} = \min \set{g_{p, 1}, \ldots, g_{p, m}}$.

Last, the strong federated privacy lower bound can be obtained by applying the composition theorem again:
\[
T\bigg(\prod_{i\neq j}M_i(\bS), \prod_{i\neq j} M_i(\bSj)\bigg) =
\bigotimes_{i \neq j}  T\big(M_i(\bS), M_i(\bSj)\big) \geq g_{p, j_{\min}}^{\otimes (m-1)R}.
\]

\section{Proof of Theorem~\ref{thm:alg_1_and_2}}
\label{sec:proof_thm_alg_1_and_2}
Let $g_{p, j} = \max(f_j, 1 - \alpha - p^2)$.
By Theorem~\ref{thm:alg_1}, it holds that
\begin{equation}
\label{eq:proof_thm_alg_1_and_2_helper_1}
T\big( M_i(\bS),  M_i(\bSj)\big) \geq g_{p, j}^{\otimes R}, \; i \in [m].
\end{equation}
We can apply the CLT type of result in \cite[Theorem 3.5]{dong2019gaussian} to obtain
the asymptotic convergence of \eqref{eq:proof_thm_alg_1_and_2_helper_1}. Yet we found that taking the $1-\alpha-p^2$ component into account will give rise to a trade-off function that does not have an explicit form. Nonetheless, we can still lower bound
\begin{equation}
\label{eq:proof_thm_alg_1_and_2_helper_2}
   T\big( M_i(\bS),  M_i(\bSj)\big) \geq f_j^{\otimes R}, \; i \in [m].
\end{equation}
We then utilize the following result from \cite{dong2019gaussian} to obtain $f_j$.
\begin{lemma}[\cite{dong2019gaussian}]
Algorithm~\ref{alg:localsgd} is $C_{B_j/n_j}(G_{1/\sigma_j})^{\otimes K}$-differentially private.
\end{lemma}
Plugging $f_j = C_{B_j/n_j}(G_{1/\sigma_j})^{\otimes K}$ into Equation~\eqref{eq:proof_thm_alg_1_and_2_helper_2}, we obtain
\begin{equation}
   T\big( M_i(\bS),  M_i(\bSj)\big) \geq C_{B_j/n_j}(G_{1/\sigma_j})^{\otimes KR}, \; i \in [m].
\end{equation}
The asymptotic convergence then follows from Corollary 5.4 of \cite{dong2019gaussian}:
$C_{B_j/n_j}(G_{1/\sigma_j})^{\otimes KR} \rightarrow G_{\mu_j}$ if $\frac{B_j}{n_j}\sqrt{KR} \rightarrow c_j$ as $\sqrt{KR} \rightarrow \infty$
where
\[
\mu_j = \sqrt{2} c_j \sqrt{e^{\sigma_j^{-2}} \Phi(1.5 \sigma_j^{-1}) + 3 \Phi(-0.5 \sigma_j^{-1}) - 2}.
\]
Similar to the argument for Theorem~\ref{thm:alg_1}, we take the ``least private'' $G_{\mu_j}$ as the lower bound
for the weak federated $f$-differential privacy notion, which is $G_{\mu_{\max}}$ with $\mu_{\max} = \max \set{\mu_1, \ldots, \mu_m}$.
Likewise, the trade-off function for the strong federated privacy is $G_{\sqrt{m-1} \mu_{\max}}$.

\section{Additional Plots}

\subsection{Trade-off function \texorpdfstring{$C_p(f)$}{Cpf}}
\label{sec:func_cpf}
\begin{figure}[H]
    \centering
    \includegraphics[width=0.35\columnwidth]{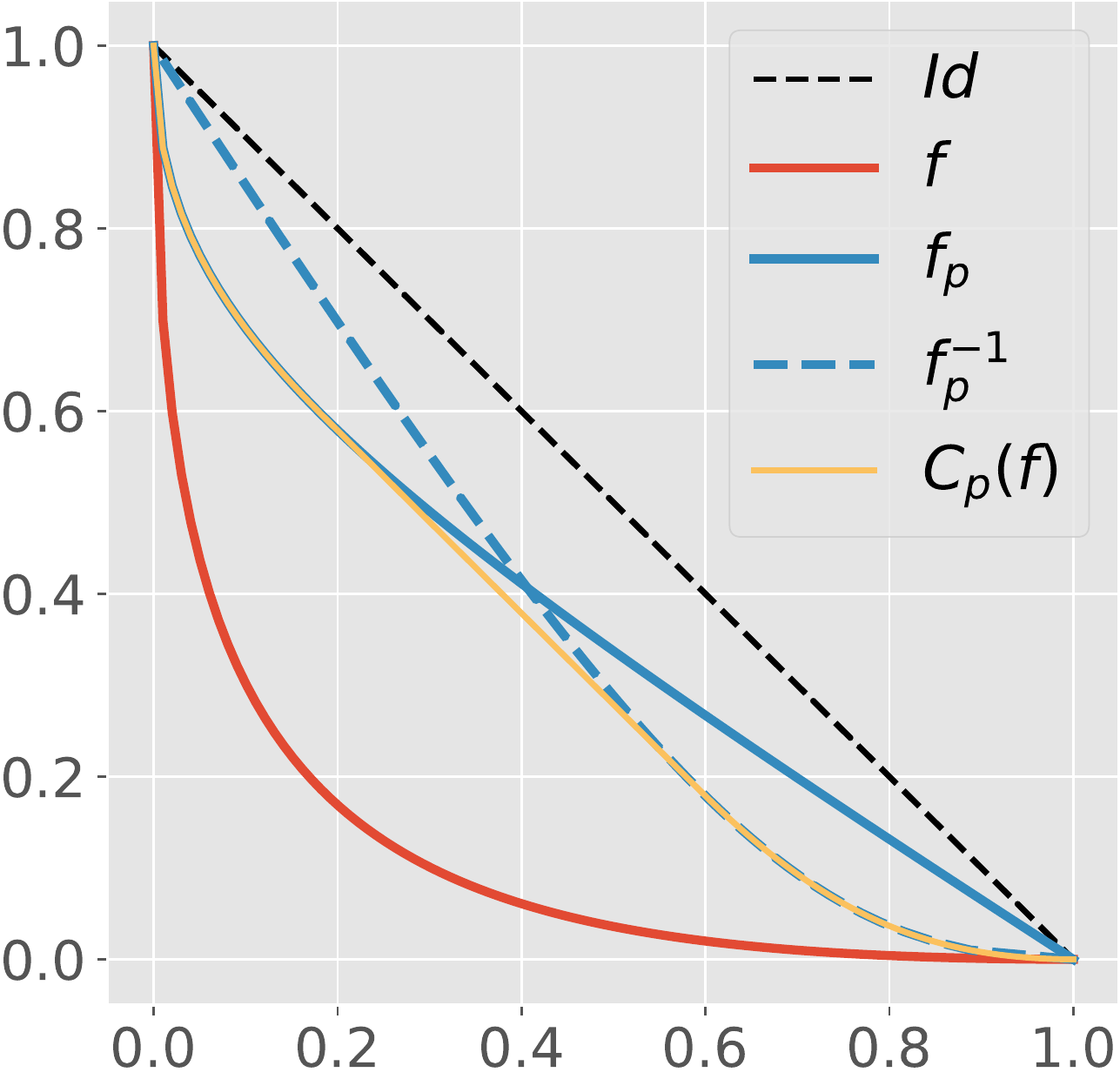}
    \caption{The trade-off function $C_p(f)$ where $f = G_{1.8}$, $p=0.35$.}
    \label{fig:cpf}
\end{figure}
Figure~\ref{fig:cpf} plots an example trade-off function $C_p(f)$ where $f$ is a GDP trade-off function $G_{1.8}$, and the sampling rate
$p=0.35$.

\subsection{Non-IID MNIST}
\begin{figure}[H]
    \centering
    \includegraphics[width=0.45\columnwidth]{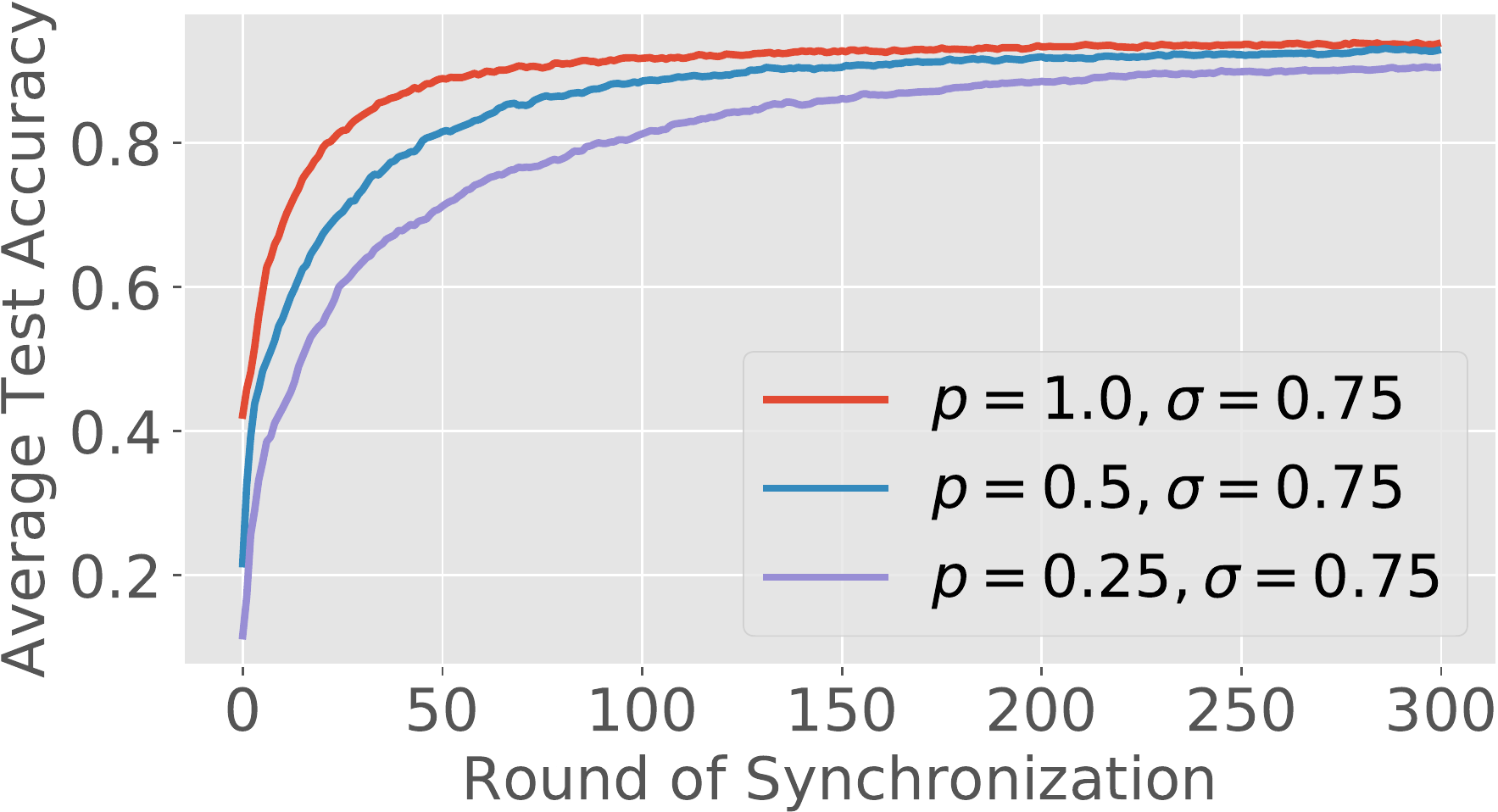}
    \caption{
    MNIST experiment: A larger sampling rate leads to faster convergence.}
    \label{fig:mnist_sampling_rate}
\end{figure}
Figure~\ref{fig:mnist_sampling_rate} plots the average test accuracy versus the number of synchronization rounds
for 3 runs with different client sampling rates in the MNIST epxeriment. It shows that the convergence is faster if we use a larger sampling rate.
The noise level is set to $\sigma=0.75$.

\subsection{Non-IID CIFAR}
%
%
\begin{figure}[H]
     \centering
     \begin{subfigure}{.4\textwidth}
         \centering
         \includegraphics[width=\textwidth]{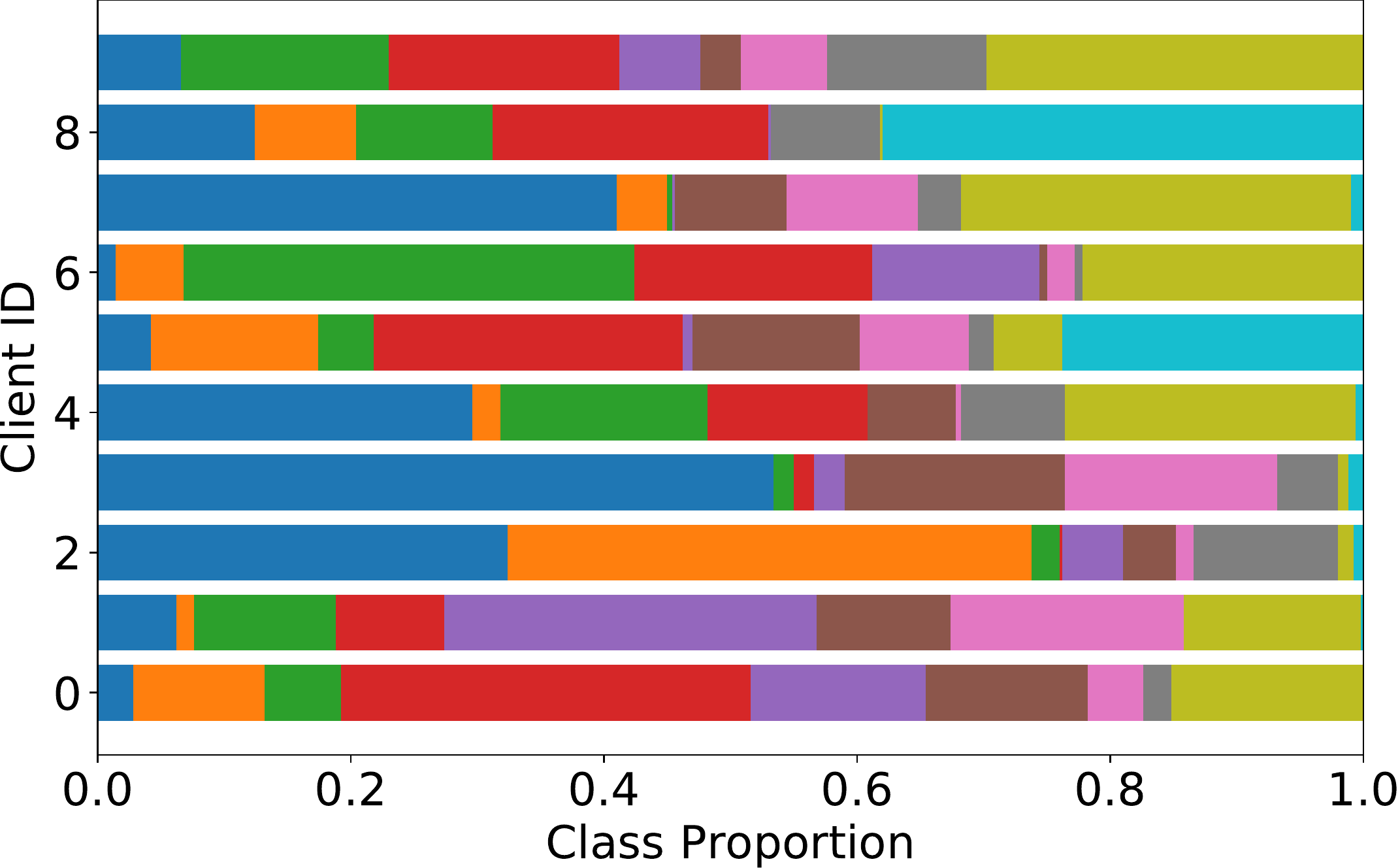}
         \caption{}
         \label{fig:cifar_class_dict}
     \end{subfigure}
     \hfill
     \begin{subfigure}{.45\textwidth}
         \centering
         \includegraphics[width=\textwidth]{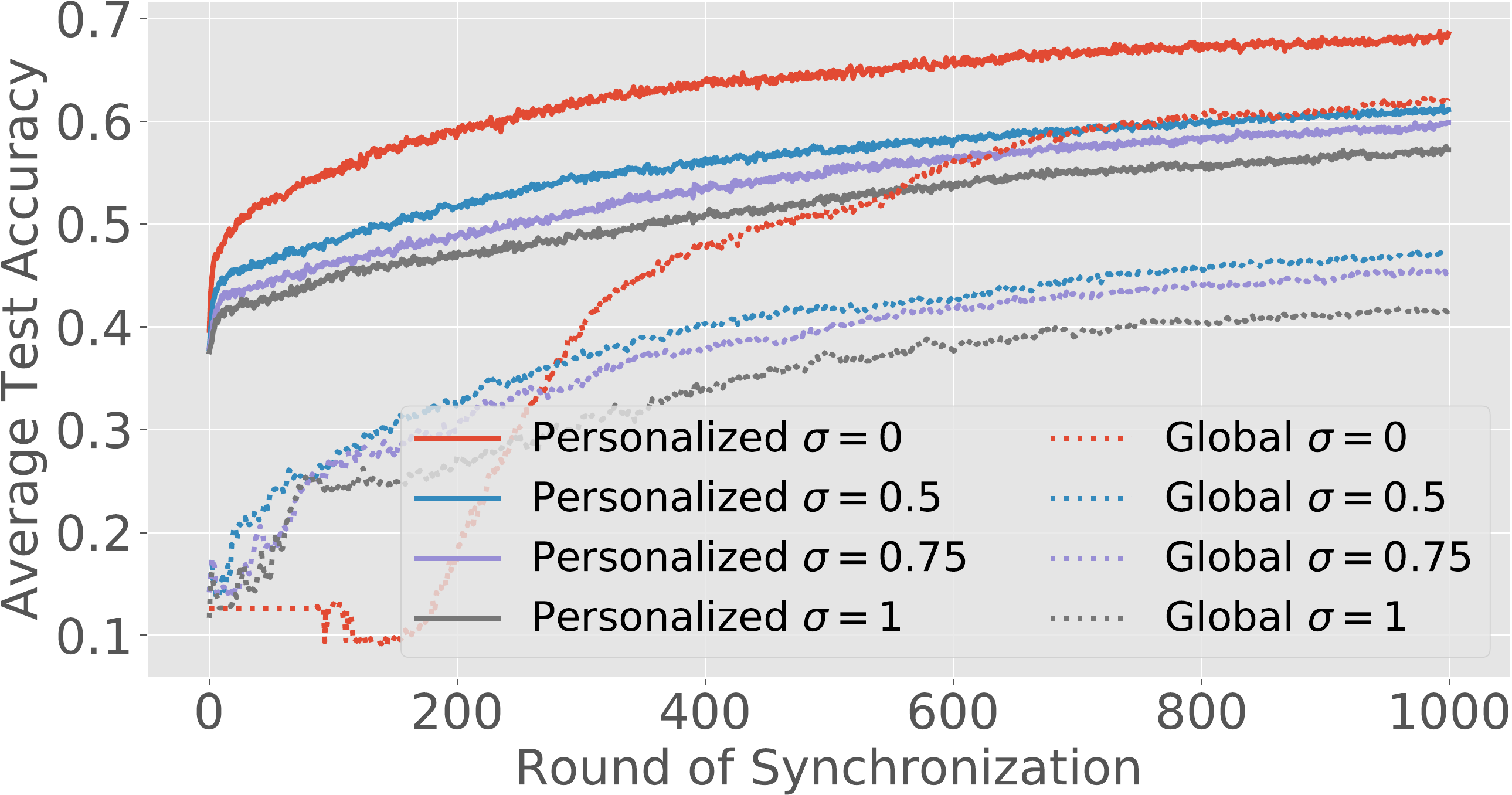}
         \caption{}
         \label{fig:cifar_convergence_p1}
     \end{subfigure}
     \caption{(a) The label class proportion for 10 randomly selected clients in the CIFAR-10 experiments. We use the Dirichlet prior with $\beta=0.5$. (b) Average top-1 test  accuracy vs synchronization rounds for the CIFAR-10 experiments. The client sampling rate is $p=1$.}
\end{figure}

To illustrate the the heterogeneity of client data distributions, Figure~\ref{fig:cifar_class_dict} plots the class proportion of the local data sets for 10 randomly selected clients. Figure~\ref{fig:cifar_convergence_p1} plots the test accuracy curve for CIFAR-10 experiment when the client sampling rate is $p=1$.

\end{document}